\documentclass[11pt]{article}

\usepackage{amsmath,amssymb,amsthm,mathrsfs}
\usepackage{graphicx}
\usepackage{hyperref}
\usepackage{aliascnt}
\usepackage{cleveref}
\usepackage[margin=1in]{geometry}
\usepackage{tikz}
\usetikzlibrary{positioning, decorations.markings, shapes.geometric}
\usepackage{xcolor}
\usepackage{marginnote}

\newtheorem{theorem}{Theorem}
\crefname{theorem}{Theorem}{Theorems}
\Crefname{theorem}{Theorem}{Theorems}

\newaliascnt{lemma}{theorem}
\newtheorem{lemma}[lemma]{Lemma}
\aliascntresetthe{lemma}
\crefname{lemma}{Lemma}{Lemmas}
\Crefname{lemma}{Lemma}{Lemmas}

\newaliascnt{proposition}{theorem}

\aliascntresetthe{proposition}
\crefname{proposition}{Proposition}{Propositions}
\Crefname{proposition}{Proposition}{Propositions}

\newaliascnt{corollary}{theorem}
\newtheorem{corollary}[corollary]{Corollary}
\aliascntresetthe{corollary}
\crefname{corollary}{Corollary}{Corollaries}
\Crefname{corollary}{Corollary}{Corollaries}

\newaliascnt{remark}{theorem}
\newtheorem{remark}[remark]{Remark}
\aliascntresetthe{remark}
\crefname{remark}{Remark}{Remarks}
\Crefname{remark}{Remark}{Remarks}

\newaliascnt{definition}{theorem}
\newtheorem{definition}[definition]{Definition}
\aliascntresetthe{definition}
\crefname{definition}{Definition}{Definitions}
\Crefname{definition}{Definition}{Definitions}

\newaliascnt{example}{theorem}
\newtheorem{example}[example]{Example}
\aliascntresetthe{example}
\crefname{example}{Example}{Examples}
\Crefname{example}{Example}{Examples}

\newcommand{\E}{\mathbb{E}}
\newcommand{\R}{\mathbb{R}}
\newcommand{\Cov}{\mathrm{Cov}}

\newcommand{\laplaceCommit}{0d9e870}
\newcommand{\leanref}[2]{%
  \marginnote{%
    \begingroup
    \hypersetup{urlbordercolor=white, pdfborder={0 0 0},
                colorlinks=false}%
    \href{https://github.com/timaeus-research/laplace/blob/\laplaceCommit/#1}{%
      \ensuremath{\square}}%
    \endgroup}}

\title{Susceptibilities and Patterning: A Primer on Linear Response in Bayesian Learning}

\author{Chris Elliott\\\textit{Timaeus} \and Daniel Murfet\\\textit{Timaeus}}

\date{\today}

\begin{document}
\maketitle

\begin{abstract}
These notes introduce the theory of susceptibilities as developed in \cite{lang2,lang3} for interpreting neural networks. The susceptibility of an observable $\phi$ to a data perturbation is defined as a derivative of a posterior expectation, which by the fluctuation--dissipation theorem equals a posterior covariance. Different choices of $\phi$ yield different objects: per-sample losses give the \emph{influence matrix} (the Bayesian influence function of \cite{singfluence1}), while component-localized observables give the \emph{structural susceptibility matrix} that pairs model components with data patterns. The susceptibility matrix is (up to a factor of $n\beta$) the Jacobian of the map from data distributions to structural coordinates; its pseudo-inverse provides a linearized solution to the \emph{patterning} problem of \cite{patterning}: finding data perturbations that produce a desired structural change. We motivate the theory from its statistical-mechanical foundations, then give a detailed exposition of susceptibilities, their empirical estimators, and their connection to the geometry of the loss landscape.
\end{abstract}

\tableofcontents

\section{Introduction} \label{sec:intro}

There is a systematic parallel between two applications of statistical modeling to a priori very different domains: the statistical mechanics of physical systems, and the analysis of the training of machine learning models.  The present notes aim to make this parallel explicit and to develop its consequences for a theoretically grounded approach to interpretability.
\begin{enumerate}
    \item \emph{Statistical mechanics}: in this setting we have a physical system with a configuration space $W$ and a real-valued energy function $H$ on $W$, whose statistics at inverse temperature $\beta$ are governed by the Boltzmann distribution with density proportional to $e^{-\beta H(w)}$.  The quantities of interest are expectation values of observables on $W$ under this distribution, together with the ways in which these respond to perturbations -- for instance, by an applied external field.
    \item \emph{Machine learning}: in this setting we have a statistical model (a family of probability distributions parameterized by a space $W$) for which we wish to minimize a loss function $L$ on $W$.  In practice one cannot work with $L$ directly, as it involves an expectation over the true data distribution, which is unknown.  Instead one has access to a finite dataset $x_1, \ldots, x_n$ and either optimizes an empirical loss $L_n$ via stochastic gradient descent, or reasons in terms of the Bayesian posterior distribution $p(w \mid x_1, \ldots, x_n) \propto p(x_1, \ldots, x_n \mid w)\pi(w)$ induced by the data and a prior $\pi$ on $W$.  When the loss is a negative log-likelihood the Bayesian posterior takes the exponential form $\propto \exp(-n L_n(w))\pi(w)$.
\end{enumerate}
Our main interest is the application of \emph{ideas} from the former setting to \emph{problems} in the latter setting.  In particular, we focus on \emph{susceptibilities}: derivatives of posterior expectation values with respect to perturbations of the data distribution.  These are a direct analog of the thermodynamic susceptibilities of statistical mechanics, and we will argue that they provide a systematic and computable tool for reading the internal structure of a trained model.

We have designed these notes to be of interest to practitioners in both domains.
\begin{enumerate}
    \item We hope that \emph{physicists} will recognize in Bayesian learning a familiar thermodynamic structure -- a configuration space, an energy function, a Boltzmann distribution, and a natural notion of perturbation -- and will find in the learning theory setting a novel class of systems to which their familiar techniques are strikingly applicable.
    \item For \emph{machine learning} scientists and practitioners, especially those with an interest in interpretability, we aim to explain why studying variations of expected values for the Bayesian posterior are a principled tool for reading the internal structure of a trained model, how they relate to and extend existing ideas such as influence functions and training data attribution, and how singular learning theory bridges the gap between population-level definitions and the empirical estimators one actually computes.
\end{enumerate}

\subsection{Structure of what follows}

We have organized the material so that readers with different backgrounds can find an efficient path through it.  Section \ref{sec:systems} develops the statistical-mechanical framework and culminates in the analogy with machine learning in \S\ref{subsec:physics_to_nn}; physicists already familiar with the Ising model and the fluctuation--dissipation theorem may wish to skim the earlier subsections and start from there.  Machine learning readers, particularly those who are not yet convinced that posterior covariances are a reasonable thing to study, should take their time with the Ising experiments of \S\ref{subsec:ising_suscept}--\ref{subsec:response_matrix}: these are intended to build the intuition that susceptibilities probe internal structure through external probes.

The remaining sections develop the theory in the machine learning setting: \S\ref{sec:setup}--\ref{sec:definition} lay out the setup, define the susceptibility, and establish the fluctuation--dissipation theorem in this setting; \S\ref{sec:geometry} develops the geometric content of susceptibilities via the Laplace approximation; \S\ref{sec:patterning} frames the susceptibility matrix as a tangent map and develops patterning as the inverse problem; and \S\ref{sec:pop_emp} addresses the passage from population-level theory to the empirical estimators used in practice.  This last section is mainly of interest to readers who wish to understand whether and how the earlier material can actually be realized.  

\section{Systems, configurations, and what we can observe}
\label{sec:systems}

The theory of susceptibilities has its roots in statistical mechanics and condensed matter physics. Before specializing to statistical models in machine learning, it is worth understanding these origins.  This will not only provide valuable intuition for how to think about susceptibilities, but in fact many of the methods used in physics may be translated directly, without modification, to the machine learning setting.

\subsection{Configurations and the Boltzmann distribution}

A \emph{system} consists of a space $W$ of \emph{configurations} and a function $H \colon W \to \R$ called the \emph{energy} or \emph{Hamiltonian}. The configurations might be parameterized continuously (for example, the positions and momenta of particles in a gas, or the weights of a neural network) or discretely (the orientations of spins on a lattice, as in a model of a magnet). The Hamiltonian assigns to each configuration a real number measuring its energy.

For reasons that are ultimately grounded in the ergodic behavior of many-body systems -- and that are beyond the scope of these notes -- it is often reasonable to model the probability that the system is in configuration $w \in W$ as proportional to $e^{-\beta H(w)}$, where $\beta > 0$ is the \emph{inverse temperature}. This is the \emph{Boltzmann distribution}:
\begin{equation}\label{eq:boltzmann}
p(w) = \frac{1}{Z}\,e^{-\beta H(w)} \text{ where } Z = \int_W e^{-\beta H(w)}\,dw.
\end{equation}
The normalization constant $Z$ is called the \emph{partition function}. At high temperature ($\beta \to 0$), all configurations are equally likely; at low temperature ($\beta \to \infty$), the distribution concentrates on the energy minima. As we will see in \S\ref{subsec:observables_perturbations} and \S\ref{subsec:fdt_theorem}, the partition function encodes, through its derivatives, essentially all thermodynamic information about the system.

\begin{example}[The Ising model] \label{subsec:ising}
 The simplest nontrivial example is the \emph{Ising model}, which describes a lattice of interacting spins. Place a spin variable $s_i \in \{+1, -1\}$ at each site $i$ of a square grid with $N = L \times L$ sites. A configuration is an assignment of a spin to every site: $w = (s_1, \ldots, s_N) \in \{-1, +1\}^N$. The Hamiltonian is
\begin{equation}\label{eq:ising_H}
H(w) = -J \sum_{i,j \text{ adjacent}} s_i s_j,
\end{equation}
where the sum runs over all pairs of nearest neighbors on the lattice and $J > 0$ is the coupling constant (we set $J = 1$ throughout). The energy is minimized when all spins are aligned ($s_i = s_j$ for all neighbors), giving two ground states: all $+1$ or all $-1$.   
\end{example}

\subsection{Observables and expectation values}
\label{subsec:observables_perturbations}

In physics, an \emph{observable} is a function $\phi: W \to \R$ on the configuration space. Its \emph{expectation value} under the Boltzmann distribution is
\[
\langle \phi \rangle = \int_W \phi(w)\,p(w)\,dw = \frac{1}{Z}\int_W \phi(w)\,e^{-\beta H(w)}\,dw.
\]
The energy $H$ itself is an example of an observable, and its expectation value also arises from the logarithmic derivative of the partition function
\[
\langle H \rangle = -\frac{\partial \log Z}{\partial \beta}.
\]

\begin{example}
In the Ising model there is a natural observable provided by the \emph{magnetization}
\[
M(w) = \sum_{i=1}^N s_i,
\]
which measures the total alignment of the spins. 
\end{example}

Expectation values are the canonical quantities to study: they are sufficient (they determine the Boltzmann distribution for a rich enough class of observables) and they probe structure (different observables reveal different aspects of the system).

\begin{remark}
\label{subsec:callen}
This perspective is articulated with particular clarity by Callen \cite{callen} in the opening chapter of his thermodynamics textbook.  Callen's starting point is an observation about what is actually observable in physics. The microscopic state of a system -- the precise configuration of $10^{23}$ particles, or the exact spin at every lattice site -- fluctuates rapidly and chaotically. What we actually measure in the laboratory are quantities that are \emph{stable}: averages over spatial regions and time intervals that are large compared to the microscopic scales but small compared to the macroscopic scales of interest. Temperature, pressure, magnetization; all of these are averages. The subject of thermodynamics, in Callen's presentation, begins with the recognition that these averaged quantities obey their own laws, independent of the microscopic details that have been averaged away.
\end{remark}

\subsection{Expectation values in the Ising model}
We return now to the Ising model: we may understand its thermodynamics by considering the expectation value of the magnetization. At high temperature ($\beta \to 0$), the Boltzmann distribution is nearly uniform over configurations, the spins point in random directions, and $\langle M \rangle \approx 0$. At low temperature ($\beta \to \infty$), the distribution concentrates on the two ground states, and $\langle |M| \rangle \approx N$.

The transition between these regimes is \emph{sharp}: there is a critical inverse temperature $\beta_c = \tfrac{1}{2}\ln(1 + \sqrt{2}) \approx 0.4407$ at which the system undergoes a \emph{phase transition}. For $\beta < \beta_c$, the system is disordered ($\langle |M| \rangle / N \approx 0$); for $\beta > \beta_c$, long-range order emerges and $\langle |M| \rangle / N > 0$. This transition is visible in \Cref{fig:phase_transition}: at low $\beta$, the lattice is a random mosaic of up and down spins; near $\beta_c$, large clusters begin to form; at high $\beta$, essentially all spins are aligned. The order parameter $\langle |M| \rangle / N$, plotted as a function of $\beta$, shows a sharp crossover near $\beta_c$.

\begin{figure}[ht]
\centering
\includegraphics[width=\textwidth]{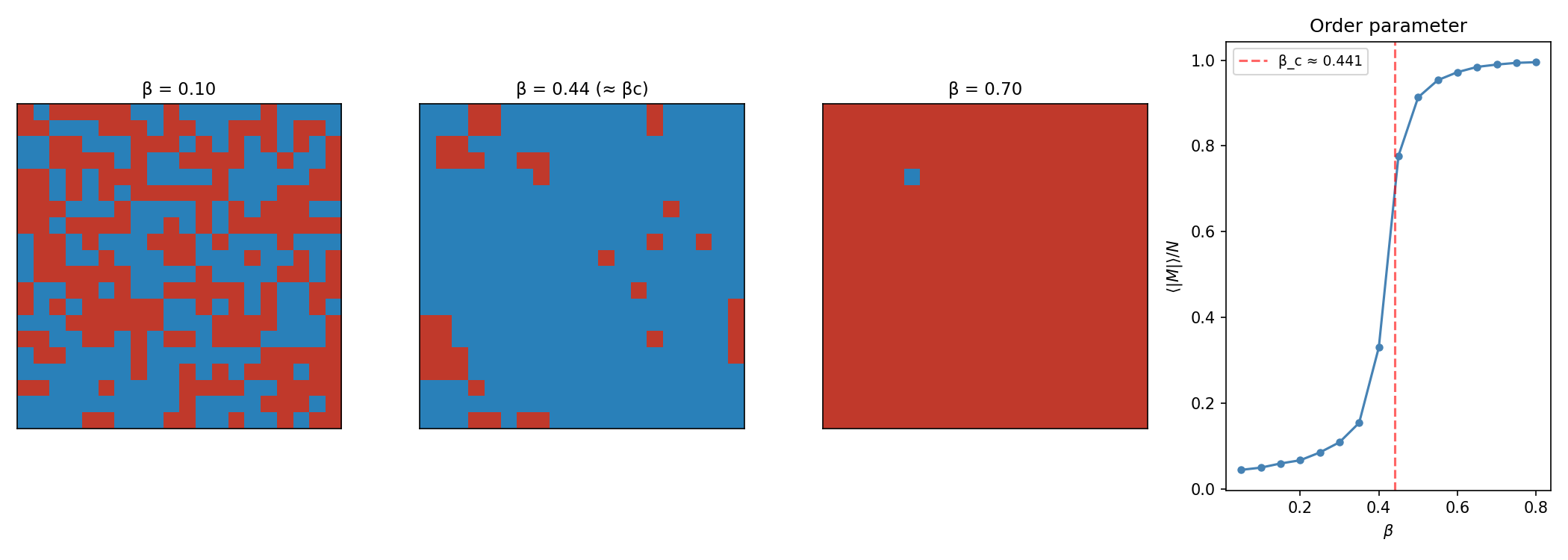}
\caption{The 2D Ising model on a $20 \times 20$ lattice with periodic boundary conditions. \textbf{Left three panels}: sample configurations drawn from the Boltzmann distribution at three values of $\beta$ (blue = spin $+1$, red = spin $-1$). At $\beta = 0.10$ (high temperature), the spins are disordered; near the critical point $\beta_c \approx 0.44$, large correlated domains appear; at $\beta = 0.70$ (low temperature), essentially all spins are aligned. \textbf{Right panel}: the order parameter $\langle |M| \rangle / N$ as a function of $\beta$, estimated from 500 Metropolis--Hastings samples after burn-in. The dashed line marks $\beta_c$. The sharp crossover is the finite-size signature of the phase transition.}
\label{fig:phase_transition}
\end{figure}

This example illustrates the general pattern: by tracking the expectation value of a single observable (the magnetization) as a function of a parameter ($\beta$), we detect a qualitative change in the internal organization of the system from disorder to order without ever inspecting individual spin configurations. The fact that the model exhibits a phase transition is not a property of individual configurations, it is a property of the Boltzmann distribution viewed as a function of the parameter $\beta$.

\subsection{Perturbations and susceptibilities}
In the real world, the Hamiltonian is never truly fixed: coupling constants depend on external conditions, and external fields can be applied. This motivates studying how expectation values change when the Hamiltonian is perturbed. Consider $H^h(w) = H(w) - h \cdot F(w)$, where $F$ is an observable and $h$ controls the perturbation strength.

For a concrete example, in the Ising model with $F = s_p$ (a single spin), the field $h > 0$ biases site $p$ toward $+1$. Because $s_p$ is coupled to its neighbors via $-Js_ps_q$, this bias propagates outward, attenuated at each step. How far it reaches is controlled by the \emph{correlation length} $\xi(\beta)$: at high temperature $\xi$ is small and the perturbation dies out quickly; near $\beta_c$ the correlation length diverges and the perturbation reaches across the system.

The expectation value of any observable $\phi$ becomes a function of $h$:
\[
\langle \phi \rangle_h = \frac{\int \phi(w)\,e^{-\beta(H(w) - hF(w))}\,dw}{\int e^{-\beta(H(w) - hF(w))}\,dw}.
\]
The \emph{susceptibility} is the first-order response:
\begin{equation}\label{eq:chi_physics}
\chi = \frac{\partial}{\partial h}\langle \phi \rangle_h\bigg|_{h=0}.
\end{equation}
As we will see in \Cref{thm:fdt}, this derivative equals a covariance $\beta\,\Cov[\phi, F]$ computed in the \emph{unperturbed} ensemble.  This is known as the \emph{fluctuation--dissipation theorem} \cite{kubo1966}.  It tells us that it is possible to study the \emph{response} to a perturbation purely as a function of expectation values in the unperturbed system.

In statistical physics, the susceptibility is the principal tool for probing the internal structure of a many-body system whose microscopic degrees of freedom cannot be inspected directly.  One controllably perturbs the system by an external source, for instance an applied magnetic field or a change in temperature, and measures the response of a macroscopic observable.  The form of the response is a window onto the system's internal organization.  For example, we have just seen that the magnetic susceptibility of the Ising model -- the response of the magnetization to a uniform external field -- diverges at the critical temperature, signaling the onset of long-range order, and the specific heat encodes the spectrum of low-lying excitations.  We may learn still richer information by studying different observables and perturbations, as we will now demonstrate.

\subsection{Example: measuring the coupling between parts of a system}
\label{subsec:ising_suscept}

To make the idea of susceptibility as a tool for interpretability concrete, we return to the Ising model and ask: \emph{can we detect which part of the lattice a given spin belongs to, just by measuring covariances?}  We will show that the answer is yes not only in theory, but also empirically.  For a more detailed account with illustrations see \cite{IsingBlog}.

\paragraph{Setup.} The Ising Hamiltonian $H = -J\sum_{i,j} s_i s_j$ sums over nearest-neighbor pairs; if we remove a site from the lattice (mask it), we remove all couplings involving that site, so a row or column of masked sites acts as a \emph{wall}, an internal boundary where the coupling constants vanish.  One can think of this as a simple model of a material with inhomogeneous structure -- a crystal with grain boundaries, a composite material with insulating layers, or a region where the interatomic spacing is too large for the exchange interaction to operate.  

For the experiment that follows consider a $20 \times 20$ Ising lattice with periodic boundary conditions and coupling $J = 1$, with a vertical column of spins down the middle masked to form a wall that blocks nearest-neighbor interactions across that column. A single site in this wall is left active, creating a small \emph{gap}. This divides the lattice into a left half ($L$, columns 0--9) and a right half ($R$, columns 11--19). Because of the periodic boundary conditions, columns 0 and 19 are also neighbors, providing a second path between the two halves -- but this path runs the long way around the lattice, so the effective coupling it mediates is weak at moderate $\beta$. We designate a single spin in the left half as the \emph{probe site} (see \Cref{fig:suscept_sweep}, top row).

\paragraph{Observables.} We define three random variables on the configuration space $\{-1,+1\}^N$, all functions of the spin configuration $w = (s_1, \ldots, s_N)$:
\begin{itemize}
\item The \emph{left magnetization} $M_L(w) = \sum_{i \in L} s_i$, the total spin of the left half.
\item The \emph{right magnetization} $M_R(w) = \sum_{i \in R} s_i$, the total spin of the right half.
\item The \emph{probe spin} $s_p(w)$, the value of the spin at the probe site.
\end{itemize}
We compute the covariances
\[
\chi_L = \Cov_\beta[M_L, s_p], \qquad \chi_R = \Cov_\beta[M_R, s_p],
\]
where $\Cov_\beta[f, g] = \langle fg \rangle_\beta - \langle f \rangle_\beta \langle g \rangle_\beta$ denotes the covariance under the Boltzmann distribution at inverse temperature $\beta$.

\paragraph{Connection to susceptibilities.} These covariances have a precise interpretation as susceptibilities. Consider perturbing the Hamiltonian by applying a small external field $h$ at the probe site alone:
\[
H^h(w) = H(w) - h \cdot s_p.
\]
The magnetization of a region $C$ under the perturbed distribution is $\langle M_C \rangle_h$. By the fluctuation--dissipation theorem (\Cref{thm:fdt}):
\begin{equation}\label{eq:ising_fdt}
\frac{\partial}{\partial h}\langle M_C \rangle_h \bigg|_{h=0} = \beta\,\Cov_\beta[M_C, s_p] = \beta\,\chi_C.
\end{equation}
The covariance $\chi_C$ measures how the magnetization of $C$ would respond to an external field at the probe site, computed entirely from thermal fluctuations in the \emph{unperturbed} system.

\paragraph{Results.} In \Cref{fig:suscept_sweep} we show the result of this experiment across a range of inverse temperatures. At each $\beta$, we draw 20{,}000 samples from the Boltzmann distribution via the Metropolis--Hastings algorithm and estimate $\chi_L$ and $\chi_R$ as sample covariances.

\begin{figure}[ht]
\centering
\includegraphics[width=\textwidth]{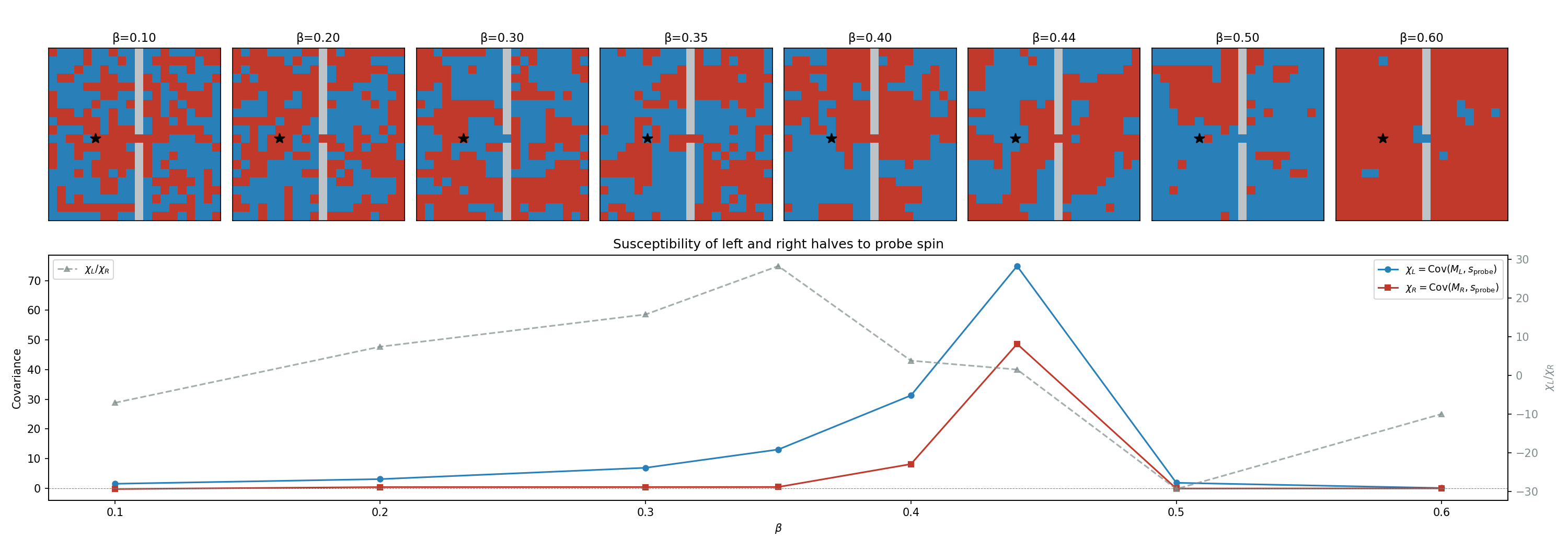}
\caption{Susceptibility of left and right halves to a probe spin in the left half, as a function of $\beta$, on a $20 \times 20$ Ising lattice with a wall separating the two halves (gap at one site). \textbf{Top row}: sample configurations at each $\beta$ (blue = $+1$, red = $-1$, gray = masked wall, $\star$ = probe site). \textbf{Bottom left}: the raw covariances $\chi_L = \Cov(M_L, s_p)$ and $\chi_R = \Cov(M_R, s_p)$. \textbf{Bottom right}: the ratio $\chi_L / \chi_R$ (plotted only where $\chi_R$ is meaningfully nonzero). The ratio peaks around $\beta \approx 0.35$, where correlations within the left half are strong but the wall still blocks coupling to the right half, and collapses toward $1$ near $\beta_c \approx 0.44$ as the diverging correlation length renders the wall ineffective.}
\label{fig:suscept_sweep}
\end{figure}

The behavior differs across three regimes. At \textbf{high temperature} ($\beta \lesssim 0.2$), both $\chi_L$ and $\chi_R$ are small (the spins are nearly independent). At \textbf{intermediate temperature} ($\beta \approx 0.3$--$0.4$), $\chi_L$ grows rapidly while $\chi_R$ remains small: the wall blocks correlations from reaching the right half, and the ratio $\chi_L/\chi_R$ peaks above 25. \textbf{Near criticality} ($\beta \approx \beta_c$), the diverging correlation length allows long-range correlations through the gap, both covariances become large, and the ratio drops toward 1. Above $\beta_c$, the spins are nearly frozen in our simulations and both covariances collapse.\footnote{The collapse above $\beta_c$ is a property of the simulation (the Metropolis chain becomes trapped in one phase sector), not of the equilibrium Gibbs measure (which is a mixture of phases with large covariances). This parallels the neural network setting, where we work with a localized posterior within a single basin rather than the full global posterior.}

This experiment illustrates the central idea that susceptibilities will play in the neural network setting: \emph{the covariance between a localized probe and a regional observable detects the coupling between the probe and the region}. When the system has structure (the wall), this coupling differs between regions, and the susceptibility reveals which region the probe belongs to -- without any direct examination of the lattice geometry.

\subsection{Example: the response matrix}
\label{subsec:response_matrix}

We can push this idea further by considering multiple probes and multiple regions simultaneously, assembling the pairwise covariances into a \emph{response matrix}.

\paragraph{Setup.} We use a $20 \times 20$ lattice with boundary walls (all edge sites masked, eliminating periodic wraparound effects) and a single internal wall: a horizontal strip of masked sites at row~10, spanning the right half of the lattice (columns~10--18). This divides the lattice into three regions (\Cref{fig:response_matrix}, left panel):
\begin{itemize}
\item \textbf{Region $A$} (blue): the left half, columns~1--9, all rows. This region is open -- it has no internal walls and shares an unobstructed boundary with both $B$ and~$C$.
\item \textbf{Region $B$} (orange): the top-right, columns~10--18, rows~1--9.
\item \textbf{Region $C$} (red): the bottom-right, columns~10--18, rows~11--18.
\end{itemize}
Regions $B$ and $C$ are separated by the wall and have no direct coupling: a spin in $B$ and a spin in $C$ are never nearest neighbors. Any interaction between $B$ and $C$ must be mediated through~$A$, which is open to both.

We place two probe spins in each region (six probes total, marked as stars in \Cref{fig:response_matrix}), and define the magnetization $M_\alpha = \sum_{i \in \alpha} s_i$ for each region $\alpha \in \{A, B, C\}$. The \emph{response matrix} is the $6 \times 3$ matrix
\[
\chi_{p\alpha} = \Cov_\beta[s_p, M_\alpha],
\]
with rows indexed by probes $p$ and columns by regions $\alpha$.

\paragraph{Results.} In \Cref{fig:response_matrix} we show the response matrix computed at $\beta = 0.44$ (near the critical temperature, where correlations are long-ranged) from 20{,}000 Metropolis--Hastings samples. The matrix has a clear block structure that reflects the lattice geometry:

\begin{figure}[ht]
\centering
\includegraphics[width=\textwidth]{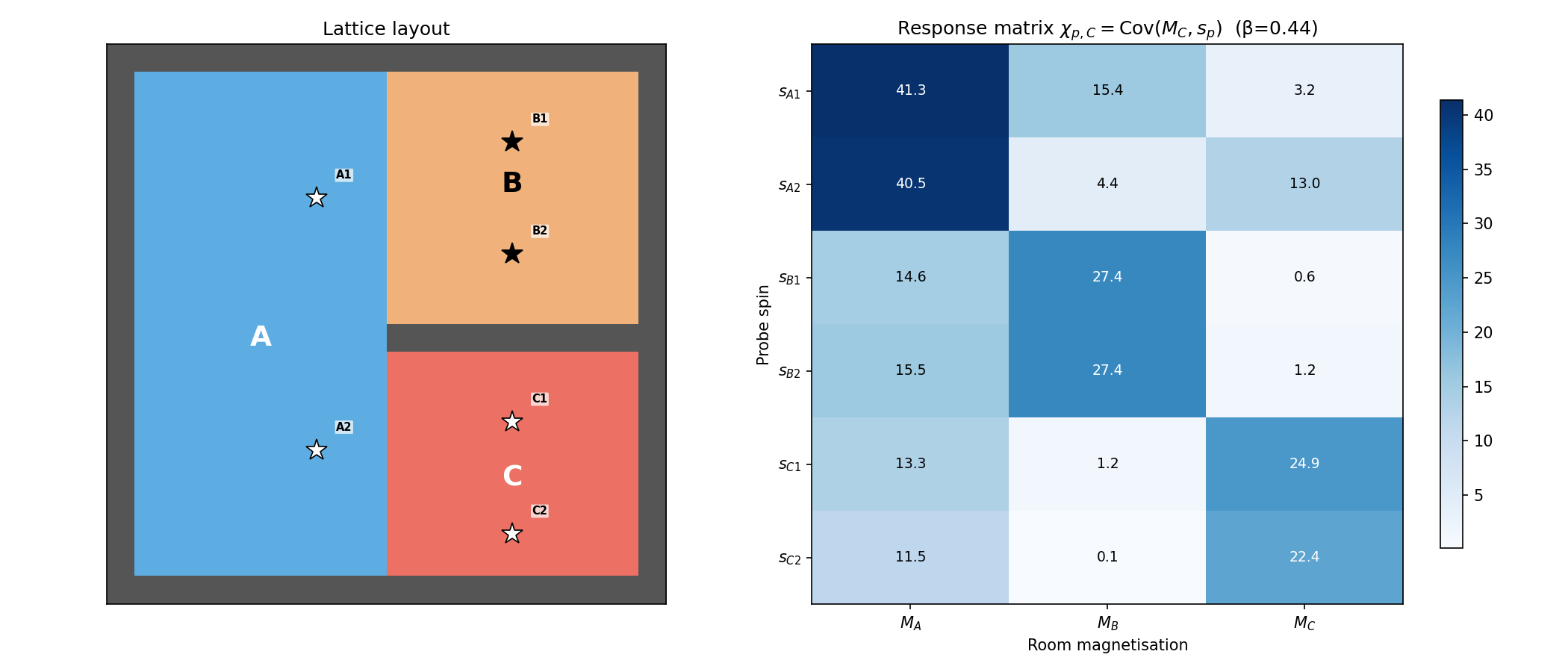}
\caption{\textbf{Left}: lattice layout showing three regions ($A$, $B$, $C$) separated by a horizontal wall on the right half. Stars mark probe spin positions. \textbf{Right}: the response matrix $\chi_{p\alpha} = \Cov_\beta[s_p, M_\alpha]$ at $\beta = 0.44$.}
\label{fig:response_matrix}
\end{figure}

\begin{itemize}
\item Every probe couples most strongly to the magnetization of its own region ($\chi_{p\alpha} \approx 25$--$41$). This is the basic signal: the susceptibility identifies which region each probe belongs to.

\item Probe $A_1$, positioned in the top half of $A$, has $\chi_{A_1, B} \approx 15$ but $\chi_{A_1, C} \approx 3$: it couples much more strongly to $B$ (which it is close to) than to $C$ (which is on the far side of the wall). Probe $A_2$, in the bottom half, shows the opposite pattern: $\chi_{A_2, C} \approx 13$ and $\chi_{A_2, B} \approx 4$. The two $A$ probes, despite belonging to the same region, have different susceptibility profiles because they ``see'' different parts of the lattice.

\item Probes in $B$ have $\chi_{p, C} \approx 0$--$1$, and probes in $C$ have $\chi_{p, B} \approx 0$--$1$. The wall between $B$ and $C$ is almost completely effective at blocking correlations. Both $B$ and $C$ probes couple moderately to $A$ ($\chi \approx 11$--$16$), since $A$ is the open region through which any $B$--$C$ interaction must be mediated.
\end{itemize}

The response matrix recovers the internal geometry from covariance measurements alone: three regions, two separated by a barrier, with a third mediating between them. The asymmetry between the $A$ probes reveals not just the room structure but the spatial layout. This is the prototype for susceptibilities in neural networks, where the ``regions'' become model components, the ``probes'' become data points, and the response matrix becomes the structural susceptibility matrix of \cite{lang2}.

\subsection{From physics to neural networks}
\label{subsec:physics_to_nn}

The structural inference program \cite{lang1,lang2,lang2.5,lang3} applies this framework to neural networks. The analogy is:

\medskip
\begin{center}
\begin{tabular}{ll}
\textbf{Physics} & \textbf{Neural networks} \\
\hline
Configuration space $W$ & Parameter space $W$ \\
Hamiltonian $H(w)$ & (Population) loss $L(w)$ \\
Boltzmann distribution $e^{-\beta H}$ & Annealed posterior $e^{-n\beta L}$ \\
External field perturbation & Data distribution perturbation \\
Observable $\phi$ (e.g.\ magnetization) & Observable $\phi$ (per-sample loss or $\phi_C$) \\
Susceptibility $\partial_h\langle \phi \rangle$ & Susceptibility $\chi(\phi; q')$ \\
\end{tabular}
\end{center}
\medskip

\noindent In \cite{lang2}, susceptibilities are defined and used to identify the roles of attention heads in small transformers: which heads are responsible for which patterns in the data. In \cite{lang3}, the per-token susceptibility vectors are clustered, yielding hundreds of interpretable groups -- the ``spectral lines'' of the model -- and a decomposition theorem shows that these clusters arise from the mode structure of the data distribution. The \emph{patterning} program \cite{patterning} inverts the framework: given a desired change in internal structure, it uses the pseudo-inverse of the susceptibility matrix to compute the optimal data perturbation.

\section{Susceptibilities in machine learning} \label{sec:definition}
We now specialize the general framework introduced in the context of statistical physics to the setting of machine learning models such as neural networks. Susceptibilities measure the first-order response of a model to perturbations of the data distribution. The general definition applies to any observable; different choices of observable yield the influence matrix and the structural susceptibility matrix.

\subsection{Setup and notation}
\label{sec:setup}
We work with a model-truth-prior triplet $(p(y|x,w), q(x,y), \pi(w))$, where $p(y|x,w)$ is a parametric model with parameters $w \in W \subseteq \R^d$, $q(x,y) = q(x)q(y|x)$ is the true data-generating distribution, and $\pi(w)$ is a prior density on $W$.

\begin{definition}[Losses]\label{def:losses}
The \emph{per-sample loss} at a data point $z = (x,y)$ is $\ell_z(w) = -\log p(y|x,w)$.  The \emph{population loss} is
\[
L(w) = \E_{z \sim q}\big[\ell_z(w)\big] = -\int q(x,y) \log p(y|x,w) \, dx \, dy.
\]
The \emph{empirical loss} on a dataset $D_n = \{(x_i, y_i)\}_{i=1}^n$ drawn i.i.d.\ from $q$ is
\[
L_n(w) = \frac{1}{n} \sum_{i=1}^n \ell_{z_i}(w).
\]
\end{definition}

\begin{definition}[Population and empirical posteriors]\label{def:posteriors}
Fix an inverse temperature $\beta > 0$ and a prior density $\pi(w)$ on $W$. The \emph{population Gibbs posterior} is
\begin{equation}\label{eq:pop_posterior}
\Pi^{\mathrm{pop}}_\beta(dw) = \frac{1}{Z^{\mathrm{pop}}_\beta} \exp\!\big\{-n\beta L(w)\big\} \, \pi(w) \, dw, \qquad Z^{\mathrm{pop}}_\beta = \int \exp\!\big\{-n\beta L(w)\big\} \, \pi(w) \, dw.
\end{equation}
The \emph{empirical posterior} (conditional on a dataset $D_n$) is
\begin{equation}\label{eq:emp_posterior}
\Pi^{\mathrm{emp}}_\beta(dw \mid D_n) = \frac{1}{Z^{\mathrm{emp}}_\beta} \exp\!\big\{-n\beta L_n(w)\big\} \, \pi(w) \, dw, \qquad Z^{\mathrm{emp}}_\beta = \int \exp\!\big\{-n\beta L_n(w)\big\} \, \pi(w) \, dw.
\end{equation}
\end{definition}

Susceptibilities are \emph{defined} using the population posterior; the empirical posterior is what we can \emph{compute}. Going forward we will use undecorated expectation values $\langle \phi \rangle$ and covariances $\Cov[\phi, \psi]$ refer to the population posterior.  Empirical quantities are marked with the superscript $\mathrm{emp}$, as in $\langle \phi \rangle^{\mathrm{emp}}$ and $\Cov^{\mathrm{emp}}[\phi, \psi]$.

\begin{remark}[Connection to the usual Bayesian posterior]
At $\beta = 1$ the empirical posterior \eqref{eq:emp_posterior} coincides with the Bayesian posterior $p(w \mid D_n) \propto p(D_n \mid w) \pi(w)$ of Bayes' theorem: since $\ell_z(w) = -\log p(y \mid x, w)$, the factor $\exp(-n L_n(w))$ equals $\prod_{i=1}^n p(y_i \mid x_i, w) = p(D_n \mid w)$.  For general $\beta > 0$ we have $\exp(-n\beta L_n(w)) = p(D_n \mid w)^\beta$, so that $\Pi^{\mathrm{emp}}_\beta$ is the \emph{tempered} variant of the standard Bayesian posterior in which the likelihood is raised to a power.
\end{remark}

\begin{definition}[Components]
A \emph{component} $C$ is a factor of parameter space. Decompose $W = U \times C$ as a product and let $w = (u, v)$ accordingly, so that $u$ denotes the parameters outside $C$ and $v$ the parameters in $C$. Given a trained parameter $w^* = (u^*, v^*)$, we think of $C$ as a part of the model whose role we wish to probe.
\end{definition}

In practice, when studying neural networks, the components are the natural building blocks of the architecture: individual attention heads, MLP layers, or layer-norm parameters.

\begin{remark}[The reference parameter $w^*$]
The theoretical development, particularly the connection to the RLCT and the resolution of singularities in \Cref{subsec:singular}, requires $w^*$ to be a local minimizer of the population loss $L$, so that $K = L - L(w^*) \ge 0$ near $w^*$.  In practice, $w^*$ is the endpoint of a training procedure rather than an exact minimizer; the gap between these is discussed in \Cref{subsec:sgd_gap}.
\end{remark}

Recall from \Cref{subsec:observables_perturbations} that an \emph{observable} is a function $\phi \colon W \to \R$ on parameter space.

\begin{definition} \label{def:struc_coord}
Given a collection of observables $\phi_1, \ldots, \phi_H$, the associated \emph{structural coordinates} of the data distribution $q$ are defined to be the vector of posterior expectations 
\[(\langle \phi_1 \rangle, \ldots, \langle \phi_H \rangle) \in \R^H.\] They summarize how the model's internal structure (as probed by the observables $\phi_j$) depends on the data. Susceptibilities measure how these coordinates respond to perturbations of $q$.
\end{definition}

\paragraph{Notation summary.} The following table collects the main symbols used in the paper.

\medskip
\begin{center}
\small
\begin{tabular}{ll}
\hline
\textbf{Symbol} & \textbf{Meaning} \\
\hline
$W$ & Parameter space \\
$w^*$ & Estimation parameter \\
$q$, $q'$ & Training (population) and probe data distributions \\
$\ell_z$ & Per-sample loss at $z = (x,y)$: $\ell_z(w) = -\log p(y|x,w)$ \\
$f_z = \ell_z - L$ & Centered per-sample loss \\
$\tilde{f}_z = \ell_z - \ell_z(w^*)$ & Log density ratio (\cite{singfluence2} convention) \\
$L = \E_q[\ell_z]$ & Population loss \\
$L_n$ & Empirical loss on dataset $D_n$ \\
$K = L - L(w^*)$ & Excess population loss (KL divergence at true params) \\
$\pi$ & Prior density on $W$ \\
$\Pi^{\mathrm{pop}}_\beta$, $\Pi^{\mathrm{emp}}_\beta$ & Population and empirical Gibbs posteriors \\
$\langle \phi \rangle$, $\Cov[\phi,\psi]$ & Population posterior expectation and covariance \\
$\langle \phi \rangle^{\mathrm{emp}}$, $\Cov^{\mathrm{emp}}$ & Empirical posterior expectation and covariance \\
$\chi(\phi; q')$ & Susceptibility of observable $\phi$ to perturbation $q \to q'$ \\
$\chi_z(\phi)$ & Per-sample susceptibility: $-\Cov[\phi, f_z]$ \\
$\phi_C = \delta(u-u^*) K$ & Component observable \\
$\tilde{\chi}^C_z$ & Renormalized component susceptibility (estimated in practice) \\
$X \in \R^{H \times D}$ & Structural susceptibility matrix \\
$I \in \R^{M \times D}$ & Influence matrix \\
$\mathcal{K}_\Pi(z,z')$ & Loss kernel: $\Cov_{w \sim \Pi}(\tilde{f}_z, \tilde{f}_{z'})$ \\
$W = U \times C$ & Decomposition into complement $\times$ component \\
\hline
\end{tabular}
\end{center}
\medskip

\subsection{Data perturbations}

Let $q$ be the training distribution and $q'$ a probe distribution. For $h \in [0,1]$, define the mixed distribution
\[
q_h = (1 - h) q + h q'.
\]
The population loss under the perturbed distribution is $L^h(w) = \E_{q_h}[\ell_z(w)]$, and we write $\Delta L(w) = \frac{\partial}{\partial h} L^h(w)\big|_{h=0} = L^{q'}(w) - L(w)$, where $L^{q'}(w) = \E_{q'}[\ell_z(w)]$.

The mixture path realizes a first-order perturbation of the data distribution.  The susceptibility defined in the next subsection depends only on the tangent vector $\eta = q' - q$ at $h = 0$ -- a signed density on the data space integrating to zero -- and is linear in $\eta$.  Linearity extends the definition to any such signed density, even those for which the affine path $q + h\eta$ does not remain non-negative for any $h > 0$.  A careful discussion of this point and of the Fr\'echet manifold structure of the space of probability density functions is given in \cite[Remark~2.15]{theory_notes}.

\subsection{Susceptibilities and the fluctuation--dissipation theorem}
\label{subsec:fdt_theorem}

Recall from \Cref{subsec:observables_perturbations} that in the statistical physics setting we remarked that the response of a Gibbs expectation to a weak external field coupled to an observable $F$ equals a covariance in the unperturbed ensemble: $\partial_h \langle \phi \rangle_h|_{h=0} = \beta\,\Cov[\phi, F]$.  This is the fluctuation--dissipation theorem \cite{kubo1966}.  We now establish the analogous statement in the machine-learning setting, with the external field replaced by a perturbation of the data distribution.

\begin{definition}[Susceptibility]\label{def:susceptibility}
Let $\phi$ be a generalized observable on $W$: either an ordinary function $\phi \colon W \to \R$ or, more generally, a distribution (such as a Dirac delta localized to a submanifold) that is well-defined when integrated against the smooth posterior density. The \emph{susceptibility} of $\phi$ to the perturbation $q \to q'$ at inverse temperature $\beta$ is
\[
\chi(\phi; q') = \frac{1}{n\beta} \left.\frac{\partial}{\partial h}\,\langle \phi \rangle_{h}\right|_{h=0},
\]
where $\langle \phi \rangle_{h}$ denotes the expectation under the population Gibbs posterior built from $L^h$ in place of $L$.
\end{definition}

\begin{theorem}[Fluctuation--dissipation; {\cite[Lemma~1]{lang2}}]\label{thm:fdt}
The susceptibility equals the negative covariance under the population posterior:
\begin{equation}\label{eq:chi_cov}
\chi(\phi; q') = -\Cov\!\big[\phi,\, \Delta L\big].
\end{equation}
\end{theorem}

\begin{proof}[Proof sketch]
Differentiate the Gibbs expectation
\[
\langle \phi \rangle_h = \frac{1}{Z(h)}\int \phi(w)\,e^{-n\beta L^h(w)}\,\pi(w)\,dw, \qquad Z(h) = \int e^{-n\beta L^h(w)}\,\pi(w)\,dw
\]
with respect to $h$.  The derivative of $e^{-n\beta L^h}$ brings down $-n\beta\,\partial_h L^h = -n\beta\,\Delta L$ (with $\Delta L = L^{q'} - L$), and the derivative of $1/Z(h)$ produces $n\beta\,\langle \Delta L \rangle_h / Z(h)$.  At $h = 0$ these combine to give
\[
\frac{\partial}{\partial h}\langle \phi \rangle_h\bigg|_{h=0} = -n\beta\big(\langle \phi\,\Delta L \rangle - \langle \phi \rangle\langle \Delta L \rangle\big) = -n\beta\,\Cov[\phi, \Delta L],
\]
and dividing by $n\beta$ gives $\chi(\phi; q') = -\Cov[\phi, \Delta L]$.
\end{proof}

This identity has a long history in Bayesian statistics in parallel to its development in statistical physics: it appears as the \emph{local case sensitivity} of Gustafson \cite{gustafson1996}, in the variational Bayes literature \cite{giordano2018, giordano2024, iba2025}, and in the neural network setting as the \emph{Bayesian influence function} of \cite{singfluence1}.  The formulation we adopt here, for general perturbations and distributional observables, is that of \cite{lang2}.

The covariance form \eqref{eq:chi_cov} is what makes susceptibilities computable: it can be \emph{estimated} by replacing $\Cov$ with $\Cov^{\mathrm{emp}}$ and then using a Monte Carlo method to sample from a distribution approximating the empirical posterior. We discuss the replacement of the population covariance by the empirical covariance in \Cref{sec:pop_emp}, and we discuss the sampling process in \Cref{sec:sgld}.

\subsection{Per-sample susceptibilities and the density interpretation}
\label{subsec:per_sample_susceptibilities}

In practice, we rarely have a single probe distribution $q'$ in mind; rather, we wish to understand the relationship between an observable and every part of the data simultaneously.  We therefore specialize to \emph{per-sample} susceptibilities, where the perturbation direction is a point mass.  These provide a \emph{universal} set of first-order perturbation directions, in the sense that the susceptibility to any probe distribution $q'$ is recovered from them by integration.  The following definition works for any observable $\phi$.

\begin{definition}[Per-sample susceptibility]\label{def:per_sample}
The \emph{per-sample susceptibility} of observable $\phi$ for the data point $(x,y)$ is
\begin{equation}\label{eq:chi_xy}
\chi_z(\phi) = -\Cov\!\big[\phi,\, \ell_z - L\big].
\end{equation}
\end{definition}

The subtraction of $L$ ensures the susceptibilities are centered, meaning that
\[\E_{z \sim q}[\chi_z(\phi)] = 0,\]
since $\E_q[\ell_z - L] = 0$ by definition of $L$.  This centering makes the per-sample susceptibilities into a \emph{density} from which the susceptibility for any perturbation can be recovered \cite[Appendix~B.2]{lang3}.  Since $\Delta L = \int (q' - q)(z)\,\ell_z\,dz$, the covariance form gives
\begin{equation}\label{eq:density_clean}
\chi(\phi; q') = \int q'(z)\,\chi_z(\phi)\,dz.
\end{equation}
That is, the susceptibility for \emph{any} probe distribution $q'$ is the $q'$-weighted average of per-sample susceptibilities. Computing them amounts to computing the susceptibility for all possible perturbations at once.

\paragraph{Point perturbations.} For $q' = (1 - \varepsilon)q + \varepsilon\,\delta_{z_0}$, the centering gives $\chi(\phi; q') = \varepsilon\,\chi_{z_0}(\phi)$: the per-sample susceptibility is, up to the perturbation strength, the response to upweighting a single data point.

\subsection{Examples of observables}
\label{subsec:natural_observables}

The susceptibility is defined for any generalized observable $\phi$; the choice of $\phi$ determines what aspect of the model we probe.  In this section we describe three families of observables that arise particularly naturally, and for two of them -- per-sample losses and component-localized losses -- we introduce the associated susceptibility \emph{matrix} obtained by pairing the observable family against a collection of per-sample perturbations.

\paragraph{Per-sample losses.} For a data point $z = (x,y)$, the per-sample loss $\ell_z(w) = -\log p(y|x,w)$ is an ordinary smooth function on $W$ (called the \emph{per-token loss} in the language model setting, where $z$ is a token-in-context). The \emph{centered per-sample loss} is
\[
f_z(w) = \ell_z(w) - L(w),
\]
the deviation of the loss on $z$ from the population mean. Centered per-sample losses integrate to zero under $q$ (by definition of $L$) and play a central role in both the density interpretation of susceptibilities (\Cref{def:per_sample}) and the loss kernel (\Cref{def:loss_kernel}).

The susceptibility of $\phi = \ell_z$ is $\chi(\ell_z; q') = -\Cov[\ell_z, \Delta L]$ and measures how the model's expected loss on $z$ responds to a shift in the data distribution. Under the empirical posterior, this is the \emph{Bayesian influence function} (BIF) of \cite{singfluence1}: the BIF of a training sample $z_i$ on a query $z_j$ is $\mathrm{BIF}(z_i, z_j) = -\beta\,\Cov^{\mathrm{emp}}[\ell_{z_j}, \ell_{z_i}]$ (up to centering conventions), the empirical version of the susceptibility with $\phi = \ell_{z_j}$ and $q' = \delta_{z_i}$.

\begin{definition}[Influence matrix]\label{def:influence_matrix}
For data points $z_1, \ldots, z_D$ and query points $z'_1, \ldots, z'_M$, the \emph{influence matrix} $I \in \R^{M \times D}$ has entries
\[
I_{mk} = \chi_{z_k}(f_{z'_m}) = -\Cov[f_{z'_m},\, f_{z_k}],
\]
where $f_z = \ell_z - L$ is the centered per-sample loss. Rows are indexed by queries, columns by data points. Under the empirical posterior this is the Bayesian influence function of \cite{singfluence1}; the population version is the loss kernel $\mathcal{K}(z'_m, z_k)$ of \cite{singfluence2}.  The influence matrix encodes the functional coupling between data points.
\end{definition}

\paragraph{The population loss.} The population loss $\phi = L(w) - L(w^*) = \E_q[\ell_z(w)] - L(w^*)$ is the observable whose expectation value (times $n\beta$) recovers the local learning coefficient $\hat{\lambda}(w^*)$ -- a singular generalization of the regular-model effective dimension $d/2$, defined and discussed in \Cref{subsec:rlct}. Its susceptibility measures the response of the overall model complexity to data perturbations.

\paragraph{Component observables.} To probe a specific component $C$ of the model, we use an analog of the population loss that only varies along the component.

\begin{definition}[Component loss]\label{def:phi_C}
Given a component $C$ with product decomposition $W = U \times C$ and reference parameter $w^* = (u^*, v^*)$, define
\begin{equation}\label{eq:phi_C}
\phi_C(w) = \delta(u - u^*)\big[L(w) - L(w^*)\big].
\end{equation}
\end{definition}

The Dirac delta $\delta(u - u^*)$ restricts to the slice where all parameters \emph{outside} $C$ are fixed at their trained values. The factor $L(w) - L(w^*)$ then measures the excess loss due to deviations in $C$. The motivation for this particular construction is discussed further in \Cref{subsec:component_losses}.

\begin{remark}[Distributional observable]
The observable $\phi_C$ is not an ordinary function on $W$: the Dirac delta makes it a distribution, well-defined only through its action inside integrals against smooth densities. Concretely, $\langle \phi_C \rangle$ reduces to a $\dim(C)$-dimensional integral over $v \in C$ with $u$ evaluated at $u^*$. In practice, $\phi_C$ cannot be evaluated at a generic posterior sample $w_t$, since such a sample will not satisfy $u = u^*$. Estimation must instead use a \emph{weight-restricted} sampling process, which samples only the component parameters $v$ while clamping $u = u^*$. This restricted sampling is discussed further in \Cref{sec:pop_emp}.
\end{remark}

\begin{definition}[Structural susceptibility matrix]\label{def:structural_matrix}
For components $C_1, \ldots, C_H$ and data points $z_1, \ldots, z_D$, the \emph{structural susceptibility matrix} $X \in \R^{H \times D}$ has entries
\[
X_{jk} = \chi_{z_k}(\phi_{C_j}) = -\Cov[\phi_{C_j},\, \ell_{z_k} - L].
\]
Rows are indexed by components, columns by data points. This is the matrix studied in \cite{lang2,lang3}: it reveals which components respond to which data patterns.  It is (up to $n\beta$) the Jacobian of the map from data distributions to structural coordinates (\Cref{sec:patterning}).
\end{definition}

\section{Susceptibilities probe geometry}
\label{sec:geometry}

Thus far we have kept the choice of observable whose response to probe very general.  We will now explain a consideration that helps us to pin down ``interesting'' observables, and in particular that motivates the utility of the component-restricted loss observables of \Cref{def:phi_C}.  The starting point is that the central object of our analysis -- the population Gibbs posterior $e^{-tL(w)}\pi(w)\,dw$ (with $t = n\beta$, as in \Cref{sec:setup}), analogous to the Boltzmann distribution from \Cref{sec:systems} -- concentrates at large $t$ near the minimum locus $W_0$ of $L$, so its structure is locally determined by the geometry of the loss near $W_0$.  This geometry ultimately controls the asymptotic behavior of Bayesian learning at large sample size -- the free energy $F_n$ and the Bayes generalization error -- via the singular learning theory of Watanabe \cite{watanabe}.  Susceptibilities are covariances under this posterior: their leading asymptotic behavior pairs the partial derivatives of the observables against those of the loss.  Different observables therefore expose different terms in the Taylor expansion of $L$, and the depth at which an observable's partial derivatives vanish at a point in $W_0$ controls how deep into the Taylor data its susceptibility reaches.  In particular this motivates considering observables that vanish, or whose partial derivatives to a specified order vanish locally in some subspace of the vanishing set $W_0$ of the loss.

We begin by presenting our analysis in the \emph{regular} setting, for which the Hessian $H = D^2 L(w^*)$\footnote{Here and below $D^k f$ denotes the $k^\text{th}$ total derivative tensor of $f$ at a point, the symmetric $k$-tensor with components $(D^k f)_{i_1 \cdots i_k} = \partial_{i_1}\cdots\partial_{i_k} f$.} is positive-definite.  In this setting the local geometry of $L$ is captured by its Taylor coefficients, and the Laplace approximation expresses susceptibilities in terms of them -- the core of \Cref{subsec:expectation_taylor} and \Cref{subsec:susceptibilities_add}.  For applications to neural networks we ultimately need to grapple with the case where the Hessian of $L$ is degenerate (not invertible), which brings in more sophisticated ideas from algebraic geometry: we generalize to this setting in \Cref{subsec:singular}, where Taylor coefficients are replaced by data of a resolution of singularities, and the rule for ``good'' observables -- vanishing along the relevant directions of the vanishing set $W_0$ -- generalizes accordingly.

\subsection{The regular case}

We're going to begin by working through a simplified setting in detail.  Suppose that $L \colon W \to \R$ has a single global minimum at $w^*$; without loss of generality we may assume $L(w^*) = 0$ by a constant translation. We will assume that $L$ is an analytic function, so its germ at $w^*$ is determined by its Taylor coefficients.  In other words, knowing all partial derivatives $\partial^\alpha L(w^*)$ completely determines the local behavior of $L$ in a neighborhood of $w^*$.

Positive-definiteness implies in particular that $w^*$ is an isolated minimum.  As we will discuss in \Cref{subsec:singular} this is a strong condition that does not generally hold in practice -- in generic examples coming from neural networks, for instance, the Hessian is non-invertible.  It will, however, be useful for doing detailed calculations to build intuition about the general case.  Under our positive-definiteness assumption the \emph{Morse lemma} tells us that we can choose smooth local coordinates near $w^*$ in which $L$ takes the standard form
\[
L(w) = \frac{1}{2}\sum_{i=1}^d \lambda_i (w_i - w^*_i)^2,
\]
where $\lambda_1, \ldots, \lambda_d > 0$ are the eigenvalues of the Hessian $H$.  In particular the Hessian completely determines the local geometry of $L$ at $w^*$, and therefore the local behavior of the Gibbs posterior at large $t$.  Our task in the rest of this subsection is to make this dependence explicit, by extracting the leading $1/t$ asymptotics of expectation values and covariances in terms of the Taylor coefficients of $L$ at $w^*$.

\begin{remark}
The fact that the Hessian completely determines the local geometry of $L$ in the regular case might seem to make probing higher derivatives like $T = D^3 L(w^*)$ unnecessary -- a change of coordinates kills them, after all.  The reason we'll care anyway is that the regular-case calculation is instructive for the more general \emph{singular} case (\Cref{subsec:singular}).  When $H$ is degenerate the higher order Taylor coefficients provide essential local information about the loss, and the susceptibility formulas we derive here are instructive for the sorts of techniques one would apply in the general case.
\end{remark}

\subsubsection{Expectation values}
\label{subsec:expectation_taylor}

We'll start by analyzing -- in the regular case -- the first non-trivial term in the asymptotic expansions of expectation values and covariances under the Gibbs posterior, and read off the resulting dependence on the Taylor coefficients of $L$ at $w^*$.  The key tool is the \emph{Laplace approximation} \cite{bender_orszag,tierney_kadane}, which Laplace introduced in the context of Bayesian inference \cite{laplace1774}. As we will see, the resulting formulas already exhibit a phenomenon that organizes the rest of the section: observables that vanish more deeply at $w^*$ probe deeper Taylor data of $L$.

Consider the expectation value of a smooth observable $\phi \colon W \to \R$:
\[\langle \phi \rangle_t = \frac{1}{Z(t)}\int_W \phi(w)\,e^{-tL(w)}\,\pi(w)\,dw,\]
where
\[Z(t) = \int_W e^{-tL(w)}\,\pi(w)\,dw\]
is the partition function.  We'll let $\pi = 1$ for simplicity. Recall that $H = D^2 L(w^*)$ is the Hessian of $L$ at the minimum, and write $\Sigma = H^{-1}$, $T = D^3 L(w^*)$ for the \emph{cubic anharmonicity} of $L$, and $(T{:}\Sigma)_i = T_{ijk}\Sigma_{jk}$ for the full contraction of $T$ against $\Sigma$ (using the summation convention).

In the regular case, the Laplace approximation provides an asymptotic expansion of integrals against the Gibbs measure as $t \to \infty$. Applying the Laplace approximation to the partition function gives
\begin{equation}\label{eq:Zt_laplace}
Z(t) \sim (2\pi)^{d/2}\, t^{-d/2}\, (\det \Sigma)^{1/2};
\end{equation}
applying it to $\int \phi(w)e^{-tL(w)}\,dw$ and dividing by \eqref{eq:Zt_laplace} gives an expansion of $\langle \phi \rangle_t$ in powers of $1/t$.  The leading term is $\phi(w^*)$: the Gibbs measure concentrates at $w^*$ and the expectation converges to the pointwise value.  This carries no information about the geometry of $L$, so to probe further we restrict to observables for which $\phi(w^*) = 0$.  This recovers the order $t^{-1}$ term.

\begin{lemma}\label{lem:laplace_exp}
Let $\phi$ be a smooth observable with $\phi(w^*) = 0$. Then
\begin{equation}\label{eq:laplace_phi}
\langle \phi \rangle_t = \frac{1}{2t}\Big[\mathrm{tr}\!\big(D^2\phi(w^*)\,\Sigma\big) - D\phi(w^*)^\top\Sigma\,(T{:}\Sigma)\Big] + O(t^{-2}).
\end{equation}
\end{lemma}

We refer to \cite[Theorem 4]{kass1990validity} for a proof. The first term couples the Hessian of $\phi$ to the inverse Hessian $\Sigma$ of $L$.  The second term couples the gradient of $\phi$ to the cubic term $T$ of $L$.

\begin{example}
In one dimension, with $L(w) = \frac{\lambda}{2}w^2 + \frac{\alpha}{6}w^3 + \frac{\gamma}{24}w^4$ (assuming that $\lambda, \gamma > 0$ and $\alpha^2 < 3\lambda\gamma$ so that $L$ has a unique global minimum at 0), for the observable $\phi = w$, this reads
\begin{equation}\label{eq:mean_anharmonic_1d}
\langle w \rangle_t = -\frac{\alpha}{2\lambda^2\, t} + O(t^{-2}).%
\end{equation}
\end{example}

When $D\phi(w^*) = 0$ as well (such as for $K(w)$, which has a minimum at $w^*$), the cubic term in \eqref{eq:laplace_phi} vanishes and
\[
\langle \phi \rangle_t = \frac{1}{2t}\,\mathrm{tr}\!\big(\Sigma\,D^2\phi(w^*)\big) + O(t^{-2}).
\]
For $\phi(w) = L(w) - L(w^*)$ we have $D^2\phi(w^*) = H$, so $\frac{1}{2}\mathrm{tr}(\Sigma H) = d/2$ and $t\langle K \rangle_t \to d/2$ as $t \to \infty$.  The constant $d/2$ is the same one that appears in the Bayesian information criterion as the effective dimension of a regular model: the leading-order asymptotic of the marginal log-likelihood is $\log p(D_n) = \log p(D_n \mid \hat w) - \tfrac{d}{2}\log n + O(1)$, and our identity is the differential form of that statement, expressing the average excess loss under the Gibbs posterior in terms of the same effective dimension.

One should take away from this calculation the following idea: \emph{to probe the local geometry of $L$ by computing expectation values, one needs to impose vanishing conditions on the leading Taylor coefficients of the observable $\phi$}.

\subsubsection{Susceptibilities: the leading order term}
\label{subsec:susceptibilities_add}
Susceptibilities involve covariances rather than single expectation values.  Applying the Laplace approximation to the covariance $\Cov_t[\phi, \psi] = \langle\phi\psi\rangle_t - \langle\phi\rangle_t\langle\psi\rangle_t$, and again assuming that the observables $\phi$ and $\psi$ vanish at $w^*$, gives the following expression for the leading order term.

\begin{lemma}\label{lem:laplace_cov}
Let $\phi, \psi$ be smooth observables both vanishing at $w^*$. Then
\begin{equation}\label{eq:laplace_cov}
\Cov_t[\phi, \psi] = \frac{1}{t}\,D\phi(w^*)^\top \Sigma\,D\psi(w^*) + O(t^{-2}).
\end{equation}
\end{lemma}

\begin{proof}
Apply \Cref{lem:laplace_exp} to the observable $\phi\psi$.  We have $\phi\psi(w^*) = 0$ and $D(\phi\psi)(w^*) = 0$, so
\begin{align*}
\langle \phi\psi \rangle_t &= \frac 1{2t} \left[\mathrm{tr}\!\big(D^2(\phi\psi)(w^*)\,\Sigma\big) - D(\phi\psi)(w^*)^\top\Sigma\,(T{:}\Sigma)\right] + O(t^{-2})  \\
&= \frac 1{2t} \mathrm{tr}\!\big(D^2(\phi\psi)(w^*)\,\Sigma\big) + O(t^{-2})  \\
&= \frac 1{t} D \phi(w^*)^\top \Sigma\, D\psi(w^*) + O(t^{-2}).
\end{align*}
Applying \Cref{lem:laplace_exp} separately to $\phi$ and to $\psi$ gives $\langle \phi \rangle_t = O(t^{-1})$ and $\langle \psi \rangle_t = O(t^{-1})$, hence $\langle \phi \rangle_t \langle \psi \rangle_t = O(t^{-2})$ and this product does not contribute to the leading order term.
\end{proof}

\begin{example}
In one dimension, with $L$ as in \eqref{eq:mean_anharmonic_1d} and $\phi = \psi = w$,
\begin{equation}\label{eq:cov_self_anharmonic_1d}
\Cov_t[w, w] = \frac{1}{\lambda\, t} + O(t^{-2}).%
\end{equation}
\end{example}

In order to gain some intuition about this leading term we can use the Morse lemma to choose local coordinates $u$ near $w^*$ in which $L$ is diagonal: $L(u) = \frac{1}{2}\sum_i \lambda_i u_i^2$, with eigenvalues $\lambda_1 \ge \lambda_2 \ge \cdots \ge \lambda_d > 0$. The inverse Hessian in this basis is $\Sigma = \mathrm{diag}(\lambda_1^{-1}, \ldots, \lambda_d^{-1})$, so
\begin{equation}\label{eq:cov_eigenvalues}
\Cov_t[\phi, \psi] = \frac{1}{t}\sum_{i=1}^d \frac{1}{\lambda_i}\,\frac{\partial\phi}{\partial u_i}(0)\,\frac{\partial\psi}{\partial u_i}(0) + O(t^{-2}).
\end{equation}
The contribution of each direction $u_i$ to the covariance is weighted by $\lambda_i^{-1}$.  Therefore the \emph{flattest} directions of the loss landscape (those with the smallest eigenvalues) dominate the covariance. This makes sense: at inverse temperature $t$, the Gibbs measure proportional to $e^{-tL}$ spreads furthest along the directions where $L$ rises most slowly, so posterior fluctuations are largest in flatter directions. Observables that vary along these flat directions will have the largest covariances.

Recall from \Cref{subsec:fdt_theorem} that the susceptibility $\chi(\phi; q') = -\Cov_t[\phi, \Delta L]$ measures the first-order response of $\langle\phi\rangle$ to a data perturbation $q \to q'$.  Specializing \Cref{lem:laplace_cov} to $\psi = \Delta L$ identifies its leading $1/t$ asymptotics with a classical statistical object.

\begin{definition}[Influence function]\label{def:influence_function}
For a smooth observable $\phi$ and a smooth perturbation $\Delta L$ of the loss, the \emph{influence function} \cite{hampel1974} of $\phi$ along $\Delta L$ is
\[
\mathrm{IF}(\phi; \Delta L) \;:=\; -D\phi(w^*)^\top\,\Sigma\,D\Delta L(w^*).
\]
\end{definition}

\begin{corollary}\label{cor:susceptibility_influence}
In the regular case, the susceptibility agrees with the influence function to leading order:
\begin{equation}\label{eq:influence_function}
\chi(\phi; q') \;=\; \frac{1}{t}\,\mathrm{IF}(\phi; \Delta L) + O(t^{-2}).%
\leanref{Laplace/Multi/CovarianceSharp.lean\#L4074}{gibbsCov\_first\_order\_rate\_sharp}
\end{equation}
\end{corollary}

Computing the influence function directly requires inverting $H$, but expressing the susceptibility as a covariance does not -- a point we will return to in \Cref{subsec:singular}, when $H$ is degenerate.

One should take away from this calculation the same idea as before: \emph{to probe the local geometry of $L$ via susceptibilities, one needs to impose vanishing conditions on the leading Taylor coefficients of the observable $\phi$}. 

\begin{example}[Per-sample susceptibilities]
For the per-sample susceptibilities of \Cref{def:per_sample}, the perturbation side is $f_z = \ell_z - L$.  Although $f_z(w^*)$ is generally nonzero, the covariance $\Cov_t[\phi, f_z]$ depends on $f_z$ only up to an additive constant, so we may apply \Cref{lem:laplace_cov} to $f_z - f_z(w^*)$ (which does vanish at $w^*$) without changing the value.  Since $w^*$ is a local minimum of $L$, we have $D L(w^*) = 0$, so $D f_z(w^*) = D \ell_z(w^*)$.  Substituting into \eqref{eq:cov_eigenvalues}:
\begin{equation}\label{eq:per_sample_eigenmodes}
\Cov_t[\phi, f_z] = \frac{1}{t}\sum_{i=1}^d \frac{1}{\lambda_i}\,\frac{\partial\phi}{\partial u_i}(w^*)\,\frac{\partial\ell_z}{\partial u_i}(w^*) + O(t^{-2}).%
\end{equation}
The data point $z$ couples strongly to eigenmode $i$ when $\partial_i \ell_z(w^*)$ is large: perturbing $w$ in direction $u_i$ changes the loss on $z$ significantly. But this contribution is weighted by $\lambda_i^{-1}$, so it matters most when $\lambda_i$ is small. The interpretation is that direction $u_i$ is one where the model can change its prediction on $z$ without much cost to overall performance (since $L$ is flat in that direction). The posterior fluctuates freely along flat directions, and if $z$'s loss varies there, the resulting covariance is large.

Conversely, a data point that couples primarily to directions where $\lambda_i$ is large has its susceptibility suppressed: even if $\partial_i\ell_z$ is large, the posterior barely fluctuates in that direction because the loss rises steeply, so there is little covariance to measure.
\end{example}

\subsubsection{Next-to-leading term}

The next term in the expansion brings in the higher order Taylor coefficients of $L$, in particular the cubic term $T = D^3 L(w^*)$.  This term will be the leading term for observables whose first derivative vanishes at $w^*$; this includes the population loss observable $K(w) = L(w) - L(w^*)$, and more generally the component loss observables $\phi_C$.  Still setting $\pi=1$ for simplicity, the full formula at order $t^{-2}$ is the following.

\begin{lemma}[Covariance at order $t^{-2}$]\label{lem:laplace_cov2}
Let $\phi$ and $\psi$ be smooth observables such that $\phi(w^*) = \psi(w^*) = 0$ and $D\phi(w^*) = 0$. Write
\[
A = D^2\phi(w^*),\quad \Phi = D^3\phi(w^*),\quad b = D\psi(w^*),\quad B = D^2\psi(w^*).
\]
Then
\begin{align}\label{eq:cov_full_t2}
\Cov_t[\phi, \psi] = \frac{1}{t^2}\bigg[&\tfrac{1}{2}\,\mathrm{tr}(A\Sigma B\Sigma) + \tfrac{1}{2}\,(\Sigma b)^\top(\Phi{:}\Sigma) \notag \\
&- \tfrac{1}{2} b^\top\Sigma A\Sigma\,(T{:}\Sigma) - \tfrac{1}{2}\,(\Sigma b)^\top\bigl(T{:}(\Sigma A\Sigma)\bigr)\bigg] + O(t^{-3}),
\end{align}
where each term is read as follows: $A\Sigma B\Sigma$ is the matrix product of two Hessians and two copies of $\Sigma$, and $\mathrm{tr}$ is the matrix trace; the contraction of a symmetric 3-tensor $K$ against a matrix $M$ is the vector
\[
(K{:}M)_i := \sum_{j,k} K_{ijk}\,M_{jk},
\]
so that $\Phi{:}\Sigma$, $T{:}\Sigma$, and $T{:}(\Sigma A \Sigma)$ are vectors; and $u^\top v = \sum_i u_iv_i$, $u^\top M w = \sum_{i,j} u_i M_{ij} w_j$ are the standard Euclidean dot product and bilinear pairing.
\end{lemma}

A proof is given in Appendix~\ref{app:laplace_cubic}.

\begin{example}
In one dimension, with $L = \frac{\lambda}{2}x^2 + \frac{\alpha}{6}x^3 + \frac{\gamma}{24}x^4$ (where $\lambda, \gamma > 0$ and $\alpha^2 < 3\lambda\gamma$ as in \eqref{eq:mean_anharmonic_1d}), $\phi = x^2$ and $\psi = x$, the lemma reduces to
\begin{equation}\label{eq:cov_anharmonic_1d}
\Cov_t[x^2, x] = -\frac{2\alpha}{\lambda^3 t^2} + O(t^{-3}).
\end{equation}
\end{example}

\begin{example}[Translated losses]\label{ex:translated_losses}
Take $\phi = K = L - L(w^*)$ (vanishing to second order at $w^*$) and $\psi = f_z = \ell_z - L$.  The latter has $\psi(w^*) = \ell_z(w^*) - L(w^*) \neq 0$ in general, but additive constants of $\psi$ drop out of the covariance, so we may apply \Cref{lem:laplace_cov2} to $\psi - \psi(w^*)$ in place of $\psi$.  The formula simplifies because $A = D^2\phi(w^*) = H$ (so $\Sigma A \Sigma = \Sigma$) and $\Phi = D^3\phi(w^*) = T$: two of the four terms of \eqref{eq:cov_full_t2} cancel and the remaining two give
\begin{equation}\label{eq:llc_suscept}
\Cov_t[K, f_z] = \frac{1}{2t^2}\,\mathrm{tr}(D^2\ell_z(w^*)\Sigma - I) - \frac{1}{2t^2}\,(\Sigma D\ell_z(w^*))^\top(T{:}\Sigma) + O(t^{-3}).
\end{equation}
In the diagonal basis this reads
\[
\Cov_t[K, f_z] = \frac{1}{2t^2}\sum_i \frac{\partial_i^2\ell_z(w^*) - \lambda_i}{\lambda_i}
  - \frac{1}{2t^2}\sum_i \frac{\partial_i\ell_z(w^*)}{\lambda_i}\sum_j \frac{T_{ijj}}{\lambda_j} + O(t^{-3}).
\]
\end{example}

The first term averages to zero against the distribution $q$ (since $\E_q[D^2\ell_z] = H$) and carries no structural information about which data points couple to which directions of the loss landscape. The summands in the second term are proportional to the product of components of the cubic tensor $T_{ijk}$ and the per-sample loss gradients $\partial_i\ell_z(w^*)$ weighted by $\lambda_i^{-1}$, so the flattest directions dominate.  If $L$ is exactly quadratic ($T = 0$), the latter contribution vanishes: at the population level, the LLC of a nondegenerate model does not respond to data perturbations at leading order.  Susceptibilities of the LLC associated to the loss observable therefore probe the \emph{higher order} structure of the loss landscape, which is exactly what makes them informative for singular models.

This suggests a canonical experiment: compute the LLC susceptibility $\Cov_t[K, f_z]$ across a collection of data points $z$ and look for nontrivial structure in the resulting vector.  When $T \neq 0$ the cubic-tensor contribution encodes which data points couple to which nonlinear directions of $L$, manifesting as clustering or low-rank patterns; when $T = 0$ this contribution vanishes.  The structured LLC susceptibilities observed empirically in \cite{lang2} are direct evidence that the cubic term $T$ is non-vanishing.

We conclude the following: \emph{to probe the cubic term $T$ (or more generally, any higher order Taylor coefficients) using susceptibilities, one must impose vanishing conditions on the gradient of the observable $\phi$ at $w^*$.}  An observable with $D\phi(w^*) \neq 0$ has its leading-order susceptibility fixed by the influence function according to \Cref{cor:susceptibility_influence}.

\subsection{The singular case}\label{subsec:singular}

In the \emph{singular} case (the generic situation for neural networks \cite{watanabe2007almost,watanabe}), the Hessian has zero eigenvalues, and the local geometry of $L$ near $w^*$ is controlled by higher-order terms in the Taylor expansion. This is where singular learning theory \cite{watanabe} enters: instead of using the Morse lemma to choose natural local coordinates we use a resolution of singularities, which (essentially) provides a set of local coordinates in which the higher order Taylor coefficients take a canonical form.  The term $d/2$ that appears in the BIC as the effective dimension of a regular model is replaced by a more refined rational invariant called the \emph{real log canonical threshold} (RLCT) of $L$ near $w^*$.

More concretely, we may standardize the singular case in two steps.  
\begin{enumerate}
\item The \emph{splitting lemma} \cite[Lemma~1]{gromoll1969} tells us that near $w^*$ there exist coordinates $(u, v) \in \R^r \times \R^{d-r}$, with $r = \mathrm{rank}\,H$, in which
\begin{equation}\label{eq:splitting}
L(u, v) = \tfrac{1}{2}\textstyle\sum_{i=1}^r \lambda_i u_i^2 + g(v), \qquad g(v) = O(|v|^3).
\end{equation}
The first term contributes a Gaussian to the posterior in the directions where $H$ is non-degenerate, and the second term contains all the singular geometry.  Unlike in the regular case the loss now admits a non-trivial vanishing locus $W_0 = \{w : L(w) = L(w^*)\}$, which is now a positive-dimensional subvariety rather than just an isolated point $\{w^*\}$.  The integral in the $u$ variables can be evaluated by the Laplace expansions of \Cref{subsec:expectation_taylor}.

\item We are left with the integral against $e^{-tg(v)}\, dv$.  This is the hard part.  We may standardize it by choosing a resolution of singularities $\pi : U \to \R^{d-r}$ on which the pullback $g \circ \pi$ is locally a monomial in the resolution coordinates.  The techniques for studying these integrals are developed in \cite{watanabe}.
\end{enumerate}   

The lessons we learned about choosing ``good'' observables in \Cref{sec:geometry} generalize accordingly: where in the regular case we asked observables to vanish to a given order \emph{at $w^*$}, in the singular case we ask them to vanish to a given order \emph{on $W_0$}.  The excess loss $K(w) = L(w) - L(w^*)$ is canonical in this sense: it vanishes on all of $W_0$ (since $L$ is constant there).  

\subsubsection{Per-sample losses}
If we make an additional assumption -- \emph{essential uniqueness}, that all parameters in $W_0$ correspond to the same distribution, i.e. $p(\cdot | w_0) = p(\cdot | w_0')$ for every $w_0, w_0' \in W_0$ -- the per-sample losses $\ell_z$ (\Cref{def:losses}) inherit the same property: $\ell_z(w) - \ell_z(w^*)$ vanishes on $W_0$.  This holds automatically when the true distribution $q$ is realizable, that is, when $q$ lies in the model class $\{p(\cdot \mid w) \colon w \in W\}$.  Per-sample losses can therefore appear naturally either as perturbation directions on the data side of the covariance, or on the observable side with $\ell_z - \ell_z(w^*)$ vanishing on $W_0$.  The resulting susceptibility, with the per-sample loss on both sides of the covariance, is the population-level loss kernel of \cite{singfluence2}.

\subsubsection{Component-restricted losses}\label{subsec:component_losses}
More generally we may consider observables that vanish only in \emph{some} directions within the vanishing locus $W_0$.  The component observable $\phi_C$ of \Cref{def:phi_C} has this property.  To isolate the geometry of $L$ within a specific component $C$, we need an observable $\phi$ that depends only on the restriction to the slice $\{u^*\} \times C$.  The natural choice is to consider the restriction of the canonical observable $K$ to $C$.  In this case $t\,\langle \phi_C \rangle_t$ is the \emph{component-refined learning coefficient} of \cite{lang1} -- a measure of the complexity attributable to component $C$.  See \cite{grammar} for a geometric account via the resolution of singularities of $L$.

\subsubsection{The RLCT}\label{subsec:rlct}

The singular analog of the regular-case identity $t\langle K \rangle_t = d/2$ is Watanabe's theorem \cite[Main Theorem 6.2]{watanabe}:
\begin{equation}\label{eq:rlct_K}
t\langle K \rangle_t = \lambda + O\!\big((\log t)^{-1}\big),
\end{equation}
where $\lambda$ is the RLCT of the loss function $L$ near $w^*$.  One can determine $\lambda$ analytically: it is the location of the leading pole, in the variable $z$, of the analytically continued zeta function
\begin{equation}\label{eq:rlct_zeta}
\zeta(z) = \int K(w)^z\,\pi(w)\,dw;
\end{equation}
the order of this pole is called the \emph{multiplicity} $m$ of the RLCT.  Both $\lambda$ and $m$ are population-level invariants of $L$.

The RLCT replaces the regular-case effective dimension $d/2$ and, like it, controls the leading-order Bayesian asymptotics.  For example, Watanabe proves singular analogues (at $\beta = 1$) of the BIC \cite[Main Theorem 6.2]{watanabe} -- this is the basis for his \emph{widely applicable Bayesian information criterion} \cite{wbic} -- and the Bayes-error formula \cite[Corollary 6.3]{watanabe},
\[
F_n = \lambda \log n - (m-1)\log\log n + O(1), \qquad \E[B_g] = \lambda/n + o(1/n),
\]
where $F_n = -\log Z^{\mathrm{emp}}_\beta$ is the free energy of the empirical posterior (\Cref{def:posteriors}), and $B_g = \mathrm{KL}\big(q \,\|\, p(\cdot \mid D_n)\big)$ is the Bayes generalization error -- the KL divergence from the truth $q$ to the Bayesian predictive distribution $p(\cdot \mid D_n) = \E_{w \sim \Pi^{\mathrm{emp}}_\beta}[p(\cdot \mid w)]$.  We access the RLCT in practice only through theorems of this kind, which connect the geometric invariant to empirical observables.

\begin{remark}
The susceptibility framework has been applied to noisy Turing machines in \cite{aixi2}, where it is shown that discrete program structure manifests as low-rank structure in the structural susceptibility matrix.  When the Turing machine decomposes into a \emph{controller} component along with two \emph{subroutine} components responsible for different classes of symbol, the off-diagonal blocks of the centered susceptibility matrix have rank at most two.  When the two subroutines are additionally equivalent up to a relabelling of symbols and states, the susceptibility matrix acquires a $\mathbb{Z}/2$-equivariance relating the two diagonal blocks.  These results provide theoretical evidence, in a well-controlled model setting, that component-restricted losses capture genuine computational structure.
\end{remark}

\begin{remark}
Much of the analysis applied to neural networks in the literature is predicated on the assumption that $L$ is nondegenerate, and thus depends on intuitions about ``curvature'' that are properties of $H$ and $\Sigma = H^{-1}$ (for instance, the classical influence function $-D\phi^\top H^{-1}D\Delta L$). As is often the case in mathematics, these intuitions do not \emph{directly} generalize to the degenerate case: when $H$ is singular, $H^{-1}$ does not exist. But the intuition can be rescued by first reformulating the same insights in terms of expectation values and covariances, which \emph{are} valid probes of the geometry of $L$ even when $H$ is degenerate. Once this transition is made (from influence functions formulated in terms of Hessians to susceptibilities formulated as covariances), the singular case becomes accessible: the covariance $\Cov[\phi, \Delta L]$ is well-defined, may be estimated in practice, and its asymptotics probe the singular geometry of the loss $L$.
\end{remark} 

The slogan to take away from this discussion extends the slogan we presented in the regular case: \emph{to probe the singular geometry of $L$ by computing expectation values, one needs observables vanishing on $W_0$ (or on a subspace thereof).}  The component observable $\phi_C$ exhibits an analog of choosing $\phi$ to vanish to second order at $w^*$  as in \Cref{lem:laplace_exp}.

\section{Susceptibilities and patterning}
\label{sec:patterning}

The preceding sections developed susceptibilities as tools for reading the internal structure of a model: given a trained network and a data distribution, a susceptibility matrix encodes which observables respond to which data patterns. The \emph{patterning} program \cite{patterning} inverts this: given a desired change in structural coordinates one wishes to find a data distribution that achieves it. This section frames both the forward problem (interpretability) and the inverse problem (patterning) in terms of a single map between tangent spaces.

\subsection{The structural coordinate map}

We take for granted that internal structure of a learned model is encoded in the local geometry of the loss landscape, and -- as discussed in the previous section -- that this geometry is probed by expectation values of observables under the posterior. Given a collection of observables $\phi_1, \ldots, \phi_H$, the vector of posterior expectations defines a map from data distributions to structural coordinates (as in \Cref{def:struc_coord}):
\begin{equation}\label{eq:mu_patterning}
\mu: \mathscr{D} \to \R^H, \qquad r \mapsto \big(\langle \phi_1 \rangle_{r},\; \ldots,\; \langle \phi_H \rangle_{r}\big),
\end{equation}
where $\mathscr{D}$ is a finite-dimensional space of data distributions of dimension $N$ -- the simplex on a discrete data space, or a chosen finite-dimensional family of perturbation directions in the continuous case -- and $\langle \phi_j \rangle_{r}$ denotes the population posterior expectation (\Cref{def:posteriors}) for the population loss coming from $r$. The map $\mu$ factors as $\mu = G \circ F$, where $F \colon \mathscr{D} \to \mathrm{Prob}(W)$ sends a data distribution to the corresponding tempered posterior and $G \colon \mathrm{Prob}(W) \to \R^H$ extracts the structural coordinates.

For a mixture perturbation $q_h = (1-h)q + hq'$ with tangent vector $dq = q' - q$, the density property \eqref{eq:density_clean} gives
\[
\frac{1}{n\beta}\,D\mu_q(dq) = X dq, \qquad X_{jk} = \chi_{z_k}(\phi_j),
\]
where $X \in \R^{H \times N}$ is the susceptibility matrix.  Since $\chi(\phi_j; q') = \frac{1}{n\beta}\frac{\partial}{\partial h}\langle \phi_j \rangle_h\big|_{h=0}$, the matrix $X$ is $\frac{1}{n\beta}$ times the Jacobian of $\mu$.  So we might write the above observation more suggestively as
\begin{equation}\label{eq:forward_tangent}
d\mu = n\beta X \, dq.
\end{equation}

\subsection{The forward problem: interpretability}

The forward problem is to read the model: given a trained network, compute the structural coordinates $\mu(q)$ and interpret them. The susceptibility matrix $X$ tells us how these coordinates respond to first-order changes in the data distribution. Its singular value decomposition 
\begin{equation} \label{eq:X_SVD}
X = \sum_\alpha \sigma_\alpha u_\alpha v_\alpha^\top
\end{equation}
 decomposes this linear response into \emph{modes}, where the right singular vectors $v_\alpha$ are coherent patterns in data space (groups of data points that affect the model in the same way) and the left singular vectors $u_\alpha$ are coherent structures in observable space (groups of observables that respond together). The singular values $\sigma_\alpha$ quantify the coupling strength. When the observables are component observables $\phi_{C_j}$, this is the basis of the ``spectroscopy'' program of \cite{lang3}: the modes are the spectral lines of the model.

\subsection{The inverse problem: patterning}
\label{subsec:patterning_inverse}

The inverse problem is to write the model: given a desired change $d\mu^* \in \R^H$ in the structural coordinates (for instance, strengthening one circuit and weakening another), find a perturbation $dq \in T_q\mathscr{D}$ that achieves it:
\begin{equation}\label{eq:patterning_inverse}
n\beta\,X\,dq = d\mu^*.
\end{equation}
Since $X$ is typically neither square nor full-rank, this system is underdetermined (many data perturbations produce the same structural change) or overdetermined (some target changes are not achievable by any first-order perturbation). The minimum-norm least-squares solution is given by applying the Moore--Penrose pseudo-inverse to $d\mu^*$ \cite{penrose1956}:
\begin{equation}\label{eq:patterning_pseudoinverse}
dq_{\mathrm{opt}} = \frac{1}{n\beta}\,X^\dagger\,d\mu^*.
\end{equation}
Using the singular value decomposition for $X$ \Cref{eq:X_SVD} we can write this concretely as
\begin{equation}\label{eq:patterning_svd}
dq_{\mathrm{opt}} = \frac{1}{n\beta}\sum_{\alpha:\,\sigma_\alpha > 0} \frac{\langle u_\alpha,\, d\mu^* \rangle}{\sigma_\alpha}\,v_\alpha.
\end{equation}
That is, we project the target onto the principal structures $u_\alpha$, scale by the inverse coupling strength $\sigma_\alpha^{-1}$, and reconstruct in data space using the corresponding principal data patterns $v_\alpha$. Small singular values are amplified, so in practice one truncates or regularizes the pseudo-inverse (we return to such practical issues in \Cref{subsec:estimator_gap}).

The conceptual point to take away is that interpretability and patterning are dual aspects of the same linear map: the susceptibility matrix $X$ is (up to scaling by $n\beta$) the Jacobian of the structural coordinate map, the SVD decomposes it into modes, and the pseudo-inverse inverts through these modes. This is the same linear-response framework used to read internal structure, now inverted to write it \cite{patterning}.

\section{Susceptibilities in practice}
\label{sec:pop_emp}

This section addresses the practical aspects of computing susceptibilities for neural networks, as implemented in \cite{lang2,lang3}.  We work through the three approximations involved in going from the population-level definition to a number computed in practice: from population to empirical posterior, from full to weight-restricted posterior for component observables, and from exact posterior expectations to SGLD samples.

\subsection{Defining susceptibilities at the population level}

The definitions in \Cref{sec:definition} involve the population loss $L(w) = \E_q[\ell_{xy}(w)]$, which requires integration over the true distribution $q$.  Since $q$ is unknown, these definitions are not directly computable.  A reader encountering the theory for the first time may wonder: why define susceptibilities in terms of a quantity we cannot measure?  The answer is the same as for the RLCT (\Cref{subsec:rlct}): the geometric content lives at the population level, where $\lambda$ is intrinsically defined as a property of $L$, and theorems connect that content to the empirical observables we actually compute.

Susceptibilities follow the same logic. The population loss $L^h(w) = (1-h)L(w) + hL^{q'}(w)$ depends smoothly on $h$, so the derivative
\[
\chi = \frac{1}{n\beta}\frac{\partial}{\partial h}\langle \phi \rangle_{\beta,h}\bigg|_{h=0}
\]
is a well-defined object: it is the derivative of the structural coordinate function on the statistical manifold of data distributions as discussed in \Cref{sec:patterning}. For the empirical loss the situation is different: $L_n$ is determined by a fixed finite sample $D_n$ and admits no continuous deformation analogous to the family $L^h$.  The natural empirical counterpart of $\chi$ is therefore a different kind of derivative -- the derivative of the empirical posterior with respect to per-sample weights on $D_n$, which yields the empirical susceptibility estimator $\hat{\chi}_{z_i}(\phi) = -\Cov^{\mathrm{emp}}[\phi, \ell_{z_i} - L_n]$ (\Cref{subsec:empirical_estimator}).  This estimator has an independent operational meaning as a per-sample weight derivative \cite{gustafson1996,giordano2018,singfluence1}.  Alternatively, as a random variable depending on $D_n$, the sense in which it approximates the population $\chi$ requires some technical analysis (outlined in \Cref{subsec:empirical_estimator} and made precise in \cite{theory_notes}).  The population-level definition is where the geometric content (described by quantities such as the RLCT) lives; the empirical estimator is how we access it.

\begin{remark} \label{subsec:sgd_gap}
We mention at this point an additional practical consideration.  The theoretical development treats $w^*$ as a local minimizer of the population loss $L$.  In practice the trained parameters $w^*$ we actually study are the endpoint of a stochastic gradient descent trajectory, only approximately a local minimizer for either the population loss $L$ or the empirical loss $L_n$.  This gap between theory and practice is not specific to susceptibilities; it is common throughout the application of singular learning theory to neural networks \cite{qd,icl1,lang1}.
\end{remark}

\subsection{The empirical susceptibility estimator}
\label{subsec:empirical_estimator}

The population susceptibility $\chi_z(\phi) = -\Cov[\phi, \ell_z - L]$ is defined under the population posterior and involves the population loss $L$. The \emph{empirical susceptibility estimator} used in \cite{lang2,lang3} replaces both by their empirical counterparts:
\begin{equation}\label{eq:empirical_chi}
\hat{\chi}_{z_i}(\phi) = -\Cov^{\mathrm{emp}}\!\big[\phi,\, \ell_{z_i} - L_n\big],
\end{equation}
where the covariance is under the empirical posterior $\Pi^{\mathrm{emp}}_\beta$ and $L_n = \frac{1}{n}\sum_j \ell_{z_j}$ is the empirical loss. This is a random variable depending on the dataset $D_n$ and the posterior samples. It can be justified from two independent directions: as an estimator of the population susceptibility (top-down), and as a per-sample weight derivative of the empirical posterior (bottom-up).

\subsubsection{Bottom-up: per-sample weight derivatives.} \label{sub:reweighting}
The covariance formula \eqref{eq:empirical_chi} has an independent meaning that does not require the population-level theory. Given a dataset $D_n = \{z_1, \ldots, z_n\}$, assign each sample a weight $\rho_i > 0$ and define the reweighted posterior
\begin{equation}\label{eq:reweighted_posterior}
p_{\boldsymbol{\rho}}(w \mid D_n) \propto \exp\!\Big(-\beta\sum_{j=1}^n \rho_j\, \ell_{z_j}(w)\Big)\,\pi(w).
\end{equation}
For the vector $\boldsymbol{\rho} = \mathbf{1}$, this is the empirical posterior. Differentiating a posterior expectation with respect to a single weight $\rho_i$ and applying the quotient rule (exactly as in the proof of \Cref{thm:fdt}) gives
\begin{equation}\label{eq:weight_derivative}
\frac{\partial}{\partial \rho_i} \E_{\boldsymbol{\rho}}[\phi(w)]\bigg|_{\boldsymbol{\rho} = \mathbf{1}} = -\beta\,\Cov^{\mathrm{emp}}\!\big[\phi,\, \ell_{z_i}\big].
\end{equation}
This is the \emph{local case sensitivity} of Gustafson \cite{gustafson1996}: the response of $\langle \phi \rangle$ to upweighting sample $z_i$. The identity \eqref{eq:weight_derivative} is the empirical counterpart of the fluctuation--dissipation theorem (\Cref{thm:fdt}): the derivative of a posterior expectation with respect to a data weight equals a posterior covariance with the corresponding log-likelihood. In the neural network setting, this weight derivative is the \emph{Bayesian influence function} (BIF) of \cite{singfluence1}.

The centering that produces the empirical susceptibility arises from restricting to probability-preserving perturbations. Reweighting by $\rho_i = 1 + s_i$ with $\sum_i s_i = 0$ changes the empirical distribution but preserves its total mass. The resulting change in $\langle \phi \rangle$ is
\[
\sum_i s_i \frac{\partial}{\partial \rho_i}\E_{\boldsymbol{\rho}}[\phi]\bigg|_{\mathbf{1}} = -\beta \sum_i s_i\,\Cov^{\mathrm{emp}}[\phi, \ell_{z_i}] = -\beta \sum_i s_i\,\Cov^{\mathrm{emp}}[\phi, \ell_{z_i} - L_n],
\]
where the centering by $L_n = \frac{1}{n}\sum_j \ell_{z_j}$ is free because $\sum_i s_i = 0$. The per-sample coefficient $-\Cov^{\mathrm{emp}}[\phi, \ell_{z_i} - L_n] = \hat{\chi}_{z_i}(\phi)$ is the empirical susceptibility \eqref{eq:empirical_chi}. The formula is \emph{exact} at the empirical level: it holds for any posterior (regular or singular, well-specified or misspecified), requires no asymptotic assumptions, and is computable from a single set of posterior samples.

\subsubsection{Top-down: estimating the population susceptibility.}
The empirical estimator \eqref{eq:empirical_chi} is, indeed, an estimator for the population susceptibility $\chi_z(\phi) = -\Cov[\phi, \ell_z - L]$ defined under $\Pi^{\mathrm{pop}}$.  The relationship between $\hat\chi$ and $\chi$ is governed by an intermediate object, the \emph{loss kernel}.  Fix a reference parameter $w^*$ and write $\tilde{f}_z(w) = \ell_z(w) - \ell_z(w^*)$ for the \emph{log density ratio} (the per-sample loss centered at $w^*$ rather than against the data distribution as in $f_z = \ell_z - L$).  For \emph{loss-linear} observables -- those of the form $\phi(w) = \int \tilde{f}_z(w)\,\alpha(dz)$ for some measure $\alpha$ on the data space -- the susceptibility factors through this kernel as a bilinear pairing.

\begin{definition}[Loss kernel]\label{def:loss_kernel}
For a posterior measure $\Pi$ on $W$, the \emph{loss kernel} is
\[
\mathcal{K}_\Pi(z, z') = \Cov_{w \sim \Pi}\!\big(\tilde{f}_z(w),\, \tilde{f}_{z'}(w)\big).
\]
We write $\mathcal{K}^{\mathrm{pop}}$ for the kernel under $\Pi^{\mathrm{pop}}$ and $\hat{\mathcal{K}}$ for the kernel under $\Pi^{\mathrm{emp}}$.
\end{definition}

The companion paper \cite{theory_notes} establishes that the empirical susceptibility estimator is \emph{consistent and asymptotically unbiased} for the population susceptibility in the SLT regime ($n\beta \to \infty$, $\beta \to 0$): $\hat{\chi}_{z_i}(\phi) - \chi_{z_i}(\phi) \to 0$ in probability, and $\E[\hat{\chi}_{z_i}(\phi)] - \chi_{z_i}(\phi) \to 0$.  The proof proceeds by establishing pointwise convergence of the empirical loss kernel $\hat{\mathcal{K}}$ to its population counterpart, plus a domination condition that controls the integration over data space; the extension to observables that are not necessarily loss-linear goes through a generalization of the loss kernel that we call the \emph{coupling kernel} (see \emph{loc.\ cit.} for details).

\subsection{Component observables in practice}

\subsubsection{Weight restriction and the renormalization gap}

For the component observable $\phi_C(w) = \delta(u - u^*)[L(w) - L(w^*)]$, the delta function introduces an additional wrinkle into the estimation procedure: writing $\Pi_{\mathrm{full}}$ for the full posterior (either $\Pi^{\mathrm{pop}}$ or $\Pi^{\mathrm{emp}}$, depending on context) and $Z_{\mathrm{full}}$ for its partition function, the expectation $\langle \phi_C \rangle$ factorizes as
\[
\E_{\Pi_{\mathrm{full}}}[\phi_C] = \frac{Z_C(u^*)}{Z_{\mathrm{full}}} \cdot \E_{\Pi_{\mathrm{ref}}}[L(u^*, v) - L(w^*)],
\]
where $Z_C(u^*) = \int e^{-n\beta L(u^*, v)}\pi(u^*, v)\,dv$ is the \emph{component partition function} and $\Pi_{\mathrm{ref}}$ is the \emph{weight-refined posterior} -- the posterior restricted to the component parameters $v$ with $u$ fixed at $u^*$.

The susceptibility of this component observable under the full posterior involves expectations for both the full posterior and its weight-refined version:
\begin{equation}\label{eq:chi_hybrid}
\chi_C(g) = -\frac{Z_C}{Z_{\mathrm{full}}}\Big(\E_{\mathrm{ref}}\big[K(u^*, v)\,\Delta L_g\big] - \E_{\mathrm{ref}}\big[K(u^*, v)\big]\,\E_{\mathrm{full}}[\Delta L_g]\Big).
\end{equation}
The inclusion of both distributions is necessary here: the delta function confines the loss term to the slice, but the term $\E_{\mathrm{full}}[\Delta L_g]$ is the global baseline arising from differentiating the full partition function.

In practice, the factor $Z_C/Z_{\mathrm{full}}$ is not computed.  Instead, one works with the \emph{renormalized susceptibility} $\tilde{\chi}_C(g) = (Z_{\mathrm{full}}/Z_C)\,\chi_C(g)$, which is the quantity directly estimated by weight-restricted SGLD \cite{lang2}.  The factor $Z_C/Z_{\mathrm{full}}$ is the posterior density on the slice $\{u^*\} \times C$; it depends on $C$ but not on the perturbation direction $g$ or the data point $(x, y)$.  For per-sample susceptibilities the relation reads
\[
\tilde{\chi}^C_{xy} = \frac{Z_{\mathrm{full}}}{Z_C}\,\chi^C_{xy},
\]
so the renormalized susceptibility differs from the population susceptibility by a $C$-dependent multiplicative constant.  While this difference is genuine, it is absorbed by the standardization procedure described in \Cref{subsubsec:standardization} below.

The theorems of \cite{theory_notes} apply to general distributional observables, including those involving a delta distribution on a component. 

\subsubsection{Standardization and the susceptibility matrix}\label{subsubsec:standardization}

The structural susceptibility matrix $\tilde{X} \in \R^{H \times D}$, with $\tilde{X}_{jk} = \tilde{\chi}^{C_j}_{x_k y_k}$, has two features that complicate direct analysis:
\begin{itemize}
\item Different components have different overall scales, partly due to the unestimated renormalization factor ($Z_{\mathrm{full}}/Z_{C_j}$ varying with $j$) and partly due to genuine differences in the magnitude of each component's response.
\item The per-sample susceptibilities are centered ($\E_q[\chi^C_{xy}] = 0$ at the population level), so the column means of the population susceptibility matrix vanish. At finite $n$, the empirical column means are small but not exactly zero.
\end{itemize}

Following \cite{lang2,lang3}, the susceptibility matrix is \emph{standardized} before downstream analysis. In \cite{lang2}, the per-sample susceptibilities are decomposed as
\[
\chi^C_{xy} = -\underbrace{\Cov\!\big[\delta(u-u^*)L(w),\, \ell_{xy}(w)\big]}_{\psi^C_{xy}} + \underbrace{\Cov\!\big[\delta(u-u^*)L(w),\, L(w)\big]}_{\text{depends only on }C},
\]
and the $C$-only term is removed by subtracting the mean across data points for each component (the column mean). This leaves the \emph{standardized susceptibility} $\psi^C_{xy}$, which depends on both $C$ and $(x,y)$.

In \cite{lang3}, a more systematic procedure is used: column-wise z-scoring (subtract mean, divide by standard deviation for each component), followed by row-wise centering (subtract mean for each data point). The column-wise z-scoring has a precise effect on the renormalization gap: since $\tilde{\chi}^C_{xy} = (Z_{\mathrm{full}}/Z_C)\,\chi^C_{xy}$, the column standard deviation is
\[
\sigma_C = \frac{Z_{\mathrm{full}}}{Z_C}\,\mathrm{std}_{xy}[\chi^C_{xy}],
\]
so dividing by $\sigma_C$ removes the factor $Z_{\mathrm{full}}/Z_C$ exactly:
\[
\frac{\tilde{\chi}^C_{xy} - \overline{\tilde{\chi}^C}}{\sigma_C} = \frac{\chi^C_{xy} - \overline{\chi^C}}{\mathrm{std}_{xy}[\chi^C_{xy}]}.
\]
The standardized matrix is therefore independent of the renormalization gap. It depends only on the \emph{shape} of each component's susceptibility profile across data points, not its overall scale. This is why the renormalization gap does not affect the conclusions drawn in \cite{lang2,lang3}: the standardization procedure absorbs it.

\subsection{SGLD estimation} \label{sec:sgld}

The posterior expectations appearing in the susceptibility are estimated by \emph{Stochastic Gradient Langevin Dynamics} (SGLD) \cite{welling}. For each component $C$, a weight-restricted SGLD chain samples the component parameters $v$ while clamping $u = u^*$:
\[
v_{t+1} = v_t - \frac{\epsilon}{2}\Big[n\beta\,D_v L_m(u^*, v_t) + \gamma(v_t - v^*)\Big] + \eta_t, \qquad \eta_t \sim \mathcal{N}(0, \epsilon),
\]
where $\epsilon > 0$ is the step size, $L_m$ is the loss on a minibatch of size $m$, and $\gamma > 0$ enforces localization around $v^*$. A separate unrestricted chain provides samples from the full empirical posterior $\Pi^{\mathrm{emp}}_\beta$, used to estimate $\E_{\mathrm{full}}[B]$.

Given $r$ samples $\{v_t\}_{t=1}^r$ from the restricted chain and $\{w_s\}_{s=1}^{r'}$ from the full chain, the renormalized per-sample susceptibility is estimated as
\begin{align*}
\hat{\tilde{\chi}}^C_{xy} \;=\; & -\frac{1}{r}\sum_{t=1}^r \big[L_n(u^*, v_t) - L_n(w^*)\big]\,\big[\ell_{xy}(u^*, v_t) - L_n(u^*, v_t)\big] \\
& + \bigg(\frac{1}{r}\sum_{t=1}^r \big[L_n(u^*, v_t) - L_n(w^*)\big]\bigg) \cdot \bigg(\frac{1}{r'}\sum_{s=1}^{r'} \big[\ell_{xy}(w_s) - L_n(w_s)\big]\bigg).
\end{align*}
This is the empirical version of the hybrid covariance \eqref{eq:chi_hybrid}: the first sum runs over the restricted chain (component-restricted samples), and the global-baseline term in the second product uses the full chain.

\subsection{Patterning in practice}
\label{subsec:patterning_practice}

The technique of patterning \eqref{eq:patterning_pseudoinverse} prescribes a perturbation of the continuous data distribution $q$.  In practice we work with a finite data batch $D_n$, and this perturbation must be realized concretely through some operation on the batch.  The natural realization is \emph{reweighting}: assigning each sample $z_i$ a weight $\rho_i$ and training against the reweighted empirical loss as in \Cref{sub:reweighting}.  This subsection works through the connection between the optimal perturbation determined by patterning in the continuous setting and its parallel given by discrete reweighting, the practical considerations that arise in implementation, and a caveat about the gap between the empirical estimator and the population pseudo-inverse it estimates, which is resolvable by ridge regularization.

\subsubsection{Batch reweighting}
Given a weight vector $\rho = (\rho_1, \ldots, \rho_n)$ with $\sum_i \rho_i = n$, define the reweighted empirical loss
\[
L^\rho_n(w) = \frac{1}{n}\sum_{i=1}^n \rho_i\,\ell_{z_i}(w).
\]
This interpolates between the original empirical loss ($\rho = \mathbf{1}$) and the empirical loss of a perturbed distribution.  Setting $\rho_i = 1 + \varepsilon\,s_i$ with $\sum_i s_i = 0$ and $\varepsilon$ a perturbation strength gives
\[
L^\rho_n = L_n + \frac{\varepsilon}{n}\sum_i s_i\,\ell_{z_i}.
\]
The perturbation direction $\Delta L_n(w) = \frac{1}{n}\sum_i s_i\,\ell_{z_i}(w) = \frac{1}{n}\sum_i s_i\,f_{z_i}(w)$ (using $\sum_i s_i = 0$ and $f_{z_i} = \ell_{z_i} - L_n$) plays the role of $\Delta L$ from the population theory.  In terms of the empirical distribution $\hat{q}_n = \frac{1}{n}\sum_i \delta_{z_i}$, the reweighted distribution is $\hat{q}^\rho = \frac{1}{n}\sum_i \rho_i \delta_{z_i}$, and the tangent vector at $\hat{q}_n$ is $dq_k = \varepsilon s_k / n$.

The resulting change in structural coordinates is, by \eqref{eq:forward_tangent},
\[
d\mu = n\beta\,\hat{X}\,dq = n\beta\,\hat{X}\,\frac{\varepsilon s}{n} = \varepsilon\beta\,\hat{X}\,s,
\]
where $\hat{X}_{ji} = \hat{\chi}_{z_i}(\phi_j)$ is the empirical susceptibility matrix evaluated on the batch.  Equivalently, we can verify this directly using the per-sample weight derivative \eqref{eq:weight_derivative}: the change in $\langle \phi_j \rangle$ from the reweighting is $\sum_i \varepsilon s_i \cdot (-\beta)\Cov^{\mathrm{emp}}[\phi_j, \ell_{z_i}] = \varepsilon\beta \sum_i s_i \hat{\chi}_{z_i}(\phi_j)$, confirming $d\mu_j = \varepsilon\beta\,(\hat{X}s)_j$.

The inverse problem for the reweighting vector $s$ is then
\begin{equation}\label{eq:batch_reweight}
\varepsilon\beta\,\hat{X}\,s = d\mu^*, \qquad s_{\mathrm{opt}} = \frac{1}{\varepsilon\beta}\,\hat{X}^\dagger\,d\mu^*.
\end{equation}
The resulting weight for sample $i$ is $\rho_i = 1 + \varepsilon\,s_i$, which upweights samples whose susceptibility profile aligns with the target and downweights those that oppose it.  In practice the perturbation strength $\varepsilon$ and the magnitude of $d\mu^*$ are not independent (one can always absorb $\varepsilon$ into $d\mu^*$); the separate notation makes explicit that the linear approximation is valid only for small $\varepsilon$.

\begin{remark}\label{rem:patterning_practice}
The per-sample weights $\rho_i = 1 + \varepsilon s_i$ can in principle be used directly as loss weights during training, with $\varepsilon$ (called \texttt{alpha} in \cite{patterning}) controlling the perturbation strength.  For large $\varepsilon$ the weights may become negative ($\rho_i < 0$) or the linear regime may break down, so in practice $\varepsilon$ is treated as a hyperparameter to be tuned.  In the induction circuit experiment of \cite[\S3]{patterning}, the loadings of each token on the leading singular vector (PC2) of the susceptibility matrix are mapped to per-token loss weights via a piecewise linear function: tokens with PC2 below a cutoff receive weight $a$ (ranging from $0$ to $4$ across experimental conditions) and those above receive weight $b = 1$, with linear interpolation in between.  In the parenthesis-balancing experiment of \cite[\S4]{patterning}, the susceptibility gap $\hat{\chi}^{EQ}_{z_i} - \hat{\chi}^{N}_{z_i}$ (proportional to $s_{\mathrm{opt}}$ under simplifying assumptions) identifies extreme samples, which are characterized qualitatively and used to design modified training distributions.  In both cases the pseudo-inverse provides the direction of the optimal perturbation; the magnitude is a separate choice.
\end{remark}

\subsubsection{Estimators for the pseudo-inverse}\label{subsec:estimator_gap}
The practical procedure just described involves two distinct operations: \emph{estimating} the susceptibility matrix $\hat{X}$ from SGLD samples, and then \emph{pseudo-inverting} the result. These do not commute: the pseudo-inverse of the estimator, $\hat{X}^\dagger$, need not be a good estimator of the population pseudo-inverse $X^\dagger$.

To see why, note that the pseudo-inverse is a discontinuous function of the matrix entries at rank boundaries: a small perturbation that changes the rank of $X$ (by pushing a singular value across zero) produces a large change in $X^\dagger$. Even away from rank boundaries, the pseudo-inverse amplifies estimation error in the small singular values by a factor of $\sigma_\alpha^{-1}$. If the $k$-th singular value of $X$ is $\sigma_k$ and the estimation error in $\hat{X}$ is $\delta$, then the corresponding contribution to $\hat{X}^\dagger$ has error of order $\delta/\sigma_k^2$, which can be much larger than $\delta$ for small $\sigma_k$.

The natural fix is to replace the pseudo-inverse with the \emph{ridge-regularized inverse}
\[
R_\lambda(\hat{X}) = \hat{X}^\top (\hat{X}\hat{X}^\top + \lambda I)^{-1},
\]
where $\lambda > 0$ is a regularization hyperparameter.  In the singular vector basis, $R_\lambda$ replaces each factor $1/\sigma_\alpha$ in the pseudo-inverse \eqref{eq:patterning_svd} by $\sigma_\alpha/(\sigma_\alpha^2 + \lambda)$: small singular values are smoothly damped rather than inverted, with the damping scale set by $\sqrt\lambda$.

The advantage over the pseudo-inverse is that $\hat{X} \mapsto R_\lambda(\hat{X})$ is real analytic and uniformly bounded in operator norm by $1/(2\sqrt\lambda)$.  These properties suffice for the theoretical results provided in \cite{theory_notes} to transfer: at every fixed $\lambda > 0$, the empirical ridge inverse $R_\lambda(\hat{X})$ is consistent and asymptotically unbiased as an estimator for $R_\lambda(X)$.  The trade-off is a population-level bias of order $\lambda/\sigma_r(X)^3$ relative to $X^\dagger$, controlled by $\lambda$ as a hyperparameter; see \cite{theory_notes} for details.

In practice, the experiments of \cite{patterning} show that the procedure produces meaningful interventions, with the patterning targets qualitatively achieved for moderate perturbation strengths.

\section{Conclusion}
\label{sec:conclusion}

We have presented the theory of susceptibilities from first principles, beginning with statistical mechanics and the Ising model, then specializing to neural networks. The central construction is the susceptibility $\chi(\phi; q') = -\Cov[\phi, \Delta L]$, defined for any observable $\phi$ and any perturbation direction $q'$. The covariance form is an instance of the fluctuation--dissipation theorem: the response of $\langle \phi \rangle$ to a data perturbation is encoded in the posterior fluctuations at the original data distribution.

Two choices of observable yield two natural matrices. Per-sample losses $\phi = \ell_{z'}$ give the \emph{influence matrix}, whose entries $-\Cov[\ell_{z'}, \ell_z - L]$ measure the functional coupling between data points; under the empirical posterior, this is the Bayesian influence function of \cite{singfluence1}. Component observables $\phi = \phi_C$ give the \emph{structural susceptibility matrix}, whose entries $-\Cov[\phi_C, \ell_z - L]$ reveal which model components respond to which data patterns. The first matrix is studied in \cite{singfluence1,singfluence2}; the second in \cite{lang2,lang2.5,lang3}. The empirical susceptibility estimator $\hat{\chi}_z(\phi) = -\Cov^{\mathrm{emp}}[\phi, \ell_z - L_n]$ has a dual justification: it is simultaneously an exact per-sample weight derivative of the empirical posterior \cite{gustafson1996,giordano2018,singfluence1} and an estimator of the population susceptibility via the results of \cite{theory_notes}.

The theory has been applied in several directions: \cite{lang2} uses the structural susceptibility matrix to identify the roles of attention heads in small transformers, \cite{lang2.5} tracks it across training to study developmental trajectories, and \cite{lang3} clusters the rows to discover interpretable groups of tokens. The \emph{patterning} program \cite{patterning} inverts the framework: given a desired change in structural coordinates, it computes the optimal data perturbation via the pseudo-inverse of the susceptibility matrix.

\appendix

\section{The Laplace expansion}
\label{app:laplace_cubic}

We compute $\Cov_t[\phi,\psi] = \langle\phi\psi\rangle_t - \langle\phi\rangle_t\langle\psi\rangle_t$ to order $t^{-2}$ by expanding each of the two terms separately.

Set $\pi = 1$.  Write $V$ for the loss in local coordinates with $V(0) = 0$.  Let $V_k$, $\phi_k$, $\psi_k$ denote the degree-$k$ homogeneous Taylor parts of $V$, $\phi$, $\psi$ at $0$, so that the lemma's hypotheses give
\begin{align*}
V(x) &= V_2(x) + V_3(x) + V_4(x) + \cdots, & V_2(x) &= \tfrac{1}{2}x^\top H x, & V_3(x) &= \tfrac{1}{6}T_{ijk}x_ix_jx_k, \\
\phi(x) &= \phi_2(x) + \phi_3(x) + \cdots, & \phi_2(x) &= \tfrac{1}{2}x^\top A x, & \phi_3(x) &= \tfrac{1}{6}\Phi_{ijk}x_ix_jx_k, \\
\psi(x) &= \psi_1(x) + \psi_2(x) + \cdots, & \psi_1(x) &= b^\top x, & \psi_2(x) &= \tfrac{1}{2}x^\top B x.
\end{align*}

\begin{lemma}\label{lem:phipsi_t2_gauss}
The posterior expected value of $\phi \psi$ is given to second order by
\begin{equation}\label{eq:phipsi_t2_gauss}
\langle \phi\psi \rangle_t = \frac{1}{t^2}\Bigl[\E[\phi_2(X)\psi_2(X)] + \E[\phi_3(X)\psi_1(X)] - \E[\phi_2(X)\psi_1(X)\,V_3(X)]\Bigr] + O(t^{-3}),
\end{equation}
where $X \sim \mathcal N(0, \Sigma)$. 
\end{lemma}

\begin{proof}
$\phi\psi$ vanishes to third order at $0$: $\phi\psi(0) = 0$, $D(\phi\psi)(0) = 0$, and $D^2(\phi\psi)(0) = 0$ by assumption on $\phi$ and $\psi$.  

The multivariate Laplace expansion presents the asymptotic series for $\langle\phi\psi\rangle_t$ as a sum over connected graphs with \emph{vertices} labeled by Taylor coefficients of the integrand at $w^*$ (one from $\phi\psi$, the rest from $V$, each of degree $\geq 3$) and \emph{edges} labeled by the \emph{propagator} $\Sigma^{ij}/t$.  With no observable insertion this is the bipartition expansion of $\log\int e^{-g}$ given in \cite[\S 3, eq.~(4)]{shun1995laplace}; in coordinates given by the Morse lemma for which $V$ is purely quadratic this is \cite[Ch.~IX, Theorem~3]{wong2001}.  A graph with one $\phi\psi$-vertex of valence $k_b \geq 3$, $\nu_V$ $V$-vertices, and $m$ propagators contributes to the order $t^{\nu_V - m}$ term in the expansion.

At $O(t^{-2})$, the constraints $\nu_V = m - 2$ and $k_b + \sum_i v_i = 2m$ (each $v_i \geq 3$) can be achieved in two ways:
\begin{itemize}
\item $k_b = 4$, $\nu_V = 0$, $m = 2$: a single quartic $\phi\psi$-vertex with both pairs of legs contracted internally.
\item $k_b = 3$, $\nu_V = 1$ (with $v_1 = 3$), $m = 3$: a cubic $\phi\psi$-vertex paired with a cubic $V$-vertex (with three propagators distributed among them, allowing self-contractions on either vertex).
\end{itemize}
By the Leibniz rule the cubic part of $\phi\psi$ is $\phi_2\psi_1$ (from $\phi^{(2)}\psi^{(1)}$) and the quartic part is $\phi_2\psi_2 + \phi_3\psi_1$.  Substitution yields the three terms in \eqref{eq:phipsi_t2_gauss}; the sign on the third term is from the $-V_3$ in the Gibbs correction.
\end{proof}

The three terms in the expansion of \Cref{lem:laplace_cov2} are presented diagrammatically in \Cref{fig:wick}.

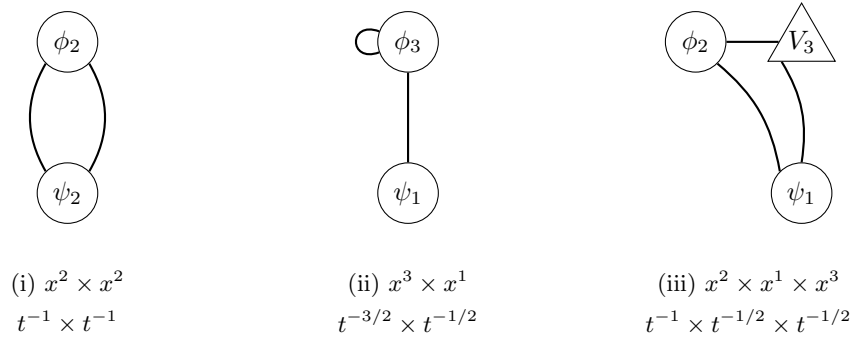
\begin{figure}[!h]
\centering
\begin{tikzpicture}[
    vertex/.style={draw, circle, minimum size=8mm, inner sep=0pt, font=\small},
    gibbs/.style={draw, regular polygon, regular polygon sides=3, minimum size=9mm, inner sep=0pt, font=\small},
    prop/.style={thick},
    label/.style={font=\footnotesize}
]

\begin{scope}[shift={(-4.5,0)}]
\node[vertex] (phi1) at (0,1) {$\phi_2$};
\node[vertex] (psi1) at (0,-1) {$\psi_2$};
\draw[prop] (phi1.south west) to[bend right=30] (psi1.north west);
\draw[prop] (phi1.south east) to[bend left=30] (psi1.north east);
\node[label] at (0,-2.2) {(i) $x^2 \times x^2$};
\node[label] at (0,-2.7) {$t^{-1} \times t^{-1}$};
\end{scope}

\begin{scope}[shift={(0,0)}]
\node[vertex] (phi2) at (0,1) {$\phi_3$};
\node[vertex] (psi2) at (0,-1) {$\psi_1$};
\draw[prop] (phi2) to[out=200, in=160, looseness=4] (phi2);
\draw[prop] (phi2.south) to (psi2.north);
\node[label] at (0,-2.2) {(ii) $x^3 \times x^1$};
\node[label] at (0,-2.7) {$t^{-3/2} \times t^{-1/2}$};
\end{scope}

\begin{scope}[shift={(4.5,0)}]
\node[vertex] (phi3) at (-0.7,1) {$\phi_2$};
\node[vertex] (psi3) at (0.7,-1) {$\psi_1$};
\node[gibbs] (v3) at (0.7,1) {$V_3$};
\draw[prop] (phi3.east) -- (v3.west);
\draw[prop] (phi3.south east) to[bend left=20] (psi3.north west);
\draw[prop] (v3.south west) to[bend left=20] (psi3.north);
\node[label] at (0,-2.2) {(iii) $x^2 \times x^1 \times x^3$};
\node[label] at (0,-2.7) {$t^{-1} \times t^{-1/2} \times t^{-1/2}$};
\end{scope}

\end{tikzpicture}
\caption{Diagrams expressing the terms that contribute to $\Cov_t[\phi, \psi]$ at order $t^{-2}$ when $D\phi(w^*) = 0$. Circles are observable vertices; the triangle is the cubic correction $V_3$ from the Gibbs measure. Edges are propagators $\Sigma_{ij} = (H^{-1})_{ij}$: each pairs two ``legs'' (factors of $X_i$). Each vertex contributes legs equal to its $x$-degree; all legs must be paired for the Gaussian integral to be nonzero. (i)~uses a higher-order term of $\psi$; (ii)~uses a higher-order term of $\phi$ (with one self-contraction); (iii)~includes the cubic correction $V_3$. Higher vanishing order of the observables forces coupling to higher-order derivatives of $L$.}
\label{fig:wick}
\end{figure}

It remains to evaluate the three Gaussian expectations in \eqref{eq:phipsi_t2_gauss} and to subtract $\langle\phi\rangle_t\langle\psi\rangle_t$.  The Gaussian expectations may be evaluated by Isserlis' theorem \cite{isserlis1918}.

The two fourth-moment expectations are
\begin{align*}
\E[\phi_2(X)\psi_2(X)] &= \tfrac{1}{4}A_{ij}B_{kl}\bigl(\Sigma_{ij}\Sigma_{kl} + \Sigma_{ik}\Sigma_{jl} + \Sigma_{il}\Sigma_{jk}\bigr) = \tfrac{1}{4}\mathrm{tr}(A\Sigma)\mathrm{tr}(B\Sigma) + \tfrac{1}{2}\mathrm{tr}(A\Sigma B\Sigma), \\
\E[\phi_3(X)\psi_1(X)] &= \tfrac{1}{6}\Phi_{ijk}b_l\bigl(\Sigma_{ij}\Sigma_{kl} + \Sigma_{ik}\Sigma_{jl} + \Sigma_{il}\Sigma_{jk}\bigr) = \tfrac{1}{2}(\Sigma b)^\top(\Phi{:}\Sigma),
\end{align*}
where the second equalities use the symmetry of $A$, $B$, $\Phi$ to combine the three pairings.  For $\E[\phi_3\psi_1]$ all three coincide; for $\E[\phi_2\psi_2]$ the pairing $\Sigma_{ij}\Sigma_{kl}$ produces $\mathrm{tr}(A\Sigma)\mathrm{tr}(B\Sigma)$ and the other two coincide and produce $\mathrm{tr}(A\Sigma B\Sigma)$.  For $\E[\phi_2\psi_1 V_3]$, we have 
\[-\E[\phi_2(X)\psi_1(X)V_3(X)] = -\tfrac{1}{12}A_{ij}b_k T_{lmn}\,\E[X_iX_jX_kX_lX_mX_n],\] 
whose $15$ summands split into three classes depending on the indices with which the $A$ indices are paired:
\begin{align*}
\text{both $A$-indices to each other ($3$ pairings):} &\quad 3\,A_{ij}\Sigma_{ij}\,b_k T_{lmn}\Sigma_{kl}\Sigma_{mn} = 3\,\mathrm{tr}(A\Sigma)(\Sigma b)^\top(T{:}\Sigma), \\
\text{one $A$-index to $b$ ($6$ pairings):} &\quad 6\,(\Sigma A\Sigma b)_l\,T_{lmn}\Sigma_{mn} = 6\,b^\top\Sigma A\Sigma(T{:}\Sigma), \\
\text{both $A$-indices to $T$ ($6$ pairings):} &\quad 6\,(\Sigma A\Sigma)_{lm}\,T_{lmn}(\Sigma b)_n = 6\,(\Sigma b)^\top\bigl(T{:}(\Sigma A\Sigma)\bigr),
\end{align*}
so 
\[
-\E[\phi_2\psi_1 V_3] = -\tfrac{1}{4}\mathrm{tr}(A\Sigma)(\Sigma b)^\top(T{:}\Sigma) - \tfrac{1}{2}b^\top\Sigma A\Sigma(T{:}\Sigma) - \tfrac{1}{2}(\Sigma b)^\top\bigl(T{:}(\Sigma A\Sigma)\bigr).
\]

The product $\langle\phi\rangle_t\langle\psi\rangle_t$ is given to order $t^{-2}$ by \Cref{lem:laplace_exp} applied to $\phi$ and $\psi$ individually:
\[
\langle\phi\rangle_t = \frac{\mathrm{tr}(A\Sigma)}{2t} + O(t^{-2}), \qquad
\langle\psi\rangle_t = \frac{\mathrm{tr}(B\Sigma) - b^\top\Sigma(T{:}\Sigma)}{2t} + O(t^{-2}),
\]
so
\[
\langle\phi\rangle_t\langle\psi\rangle_t = \frac{1}{4t^2}\mathrm{tr}(A\Sigma)\bigl[\mathrm{tr}(B\Sigma) - b^\top\Sigma(T{:}\Sigma)\bigr] + O(t^{-3}).
\]

Subtracting from \eqref{eq:phipsi_t2_gauss}, the $\tfrac{1}{4}\mathrm{tr}(A\Sigma)\mathrm{tr}(B\Sigma)$ terms cancel, as do the $\tfrac{1}{4}\mathrm{tr}(A\Sigma)(\Sigma b)^\top(T{:}\Sigma)$ terms (using $\Sigma^\top = \Sigma$).  The four surviving terms are exactly those of \eqref{eq:cov_full_t2}.

\end{document}